\documentclass[11pt]{article}
\usepackage{amsmath}
\usepackage{amssymb}
\usepackage{theorem}
\usepackage{epsfig}
\usepackage{psfrag}
\usepackage{subfigure}
\usepackage{url}
\DeclareSymbolFontAlphabet{\Bbb}{AMSb}

\newtheorem{theorem}{Theorem}[section]
\newtheorem{proposition}[theorem]{Proposition}
\newtheorem{lemma}[theorem]{Lemma}

\newtheorem{definition}[theorem]{Definition}

\theorembodyfont{\upshape\small}
\newtheorem{example}[theorem]{Example}

\theorembodyfont{\upshape}

\newcommand{\eins}{\boldsymbol{1}}

\newcommand{\N}{\Bbb{N}}    % Sonderzeichen fr
\newcommand{\R}{\Bbb{R}}    % Sonderzeichen fr
\newcommand{\Rn}{\Bbb{R}^d}
\newcommand{\E}{\Bbb{E}}    % Sonderzeichen fr

\newcommand{\ca}[1]{{\mathcal #1}}

\newcommand{\RP}[2]{{{\cal R}_{#1,P}(#2)}}
\newcommand{\RPB}[1]{{{\cal R}_{#1,P}^{*}}}

\newcommand{\Rx}[3]{{{\cal R}_{#1,#2}(#3)}}
\newcommand{\RPxB}[2]{{{\cal R}_{#1,P,#2}^*}}

\newcommand{\fP}{f_{P,\lb}}
\newcommand{\fPn}{f_{P,\lb_n}}
\newcommand{\fPnn}{f_{P_n,\lb_n}}

\newcommand{\fT}{f_{T,\lb}}
\newcommand{\fQ}{f_{Q,\lb}}
\newcommand{\fTno}{f_{T_n(\om),\lb_n}}
\newcommand{\fTnso}{f_{T_{m}(\om),\lb_{m}}}
\newcommand{\fTnon}{f_{T_n(\om),\lb}}
\newcommand{\fTnom}{f_{T_m(\om),\lb}}

\newcommand{\lclass}{{{L}_{\mathrm{class}}}}

\newcommand{\lls}{{{L}_{\mathrm{lsquares}}}}
\newcommand{\mysetminus}{\!\setminus\!}

\newcommand{\rem}[1]{}

\def \B         {{\cal B}}                % Borelsche s-Algebra

\def \Om        { \Omega }
\def \om        { \omega }
\def \lb        { \lambda }
\def \a         { \alpha }
\def \b         { \beta }

\def \e         { \varepsilon }

\def \s         { \sigma }

\def \t         { \tau }

\def \d         { \delta }

\def \p         { \varphi }

\def \P         {\Phi }

\newcommand{\Leq}{\ \leq \ }

\newcommand{\Eq}{\ = \ }
\newcommand{\Defi}{\ := \ }

\newlength{\fixboxwidth}
\newcommand{\fix}[1]{\marginpar{%
    \fbox{\parbox{\fixboxwidth}{\tiny #1}}}}

\newcommand{\BlackBox}{\rule{1.5ex}{1.5ex}}  % end of proof
{\end{list}}

\newenvironment{proof}{\begin{list}{{\bf \em Proof: }}%
{\setlength{\labelsep}{0pt}\setlength{\leftmargin}{0pt}\setlength{\labelwidth}{0pt}}\item}%
{\hfill\BlackBox\end{list}}
{\hfill\BlackBox\end{list}}
\newenvironment{proofof}[1]{\begin{list}{{\bf \em Proof of #1: }}%
{\setlength{\labelsep}{0pt}\setlength{\leftmargin}{0pt}\setlength{\labelwidth}{0pt}}\item}%
{\hfill\BlackBox\end{list}}
{\hfill\BlackBox\end{list}}

\newcommand{\Int}{\int\limits}

\DeclareMathOperator{\sign}{sign}

\newcommand{\rank}[1]{{\rm rank} \ #1}
\newcommand{\Lx}[2]{L_{#1} (#2)}
\newcommand{\Lxx}[2]{{L}_{#1} (#2)}

\newcommand{\norm}[1] {\Bigl\Vert #1 \Bigr\Vert}

\newcommand{\snorm}[1] {\Vert #1 \Vert}
\newcommand{\mnorm}[1] {\bigl\Vert #1 \bigr\Vert}

\newcommand{\inorm}[1]{\Vert #1 \Vert_{\infty}}

\newcommand{\sL}[2]{{\cal L}_{#1}(#2)}
\newcommand{\sLp}[1]{\mbox{${\cal L}_p(\mu)$ }}

\begin{document}

\author{Ingo Steinwart\footnote{Corresponding author}, Don Hush, and Clint Scovel\\
Modeling, Algorithms and Informatics Group, CCS-3\\
MS B256\\
Los Alamos National Laboratory\\
Los Alamos, New Mexico 87545, USA\\
Tel.: 001-505-665-7914\\
Fax.: 001-505-667-1126\\
\url{{ingo,dhush,jcs}@lanl.gov}\\[5mm]
}
\title{Learning from dependent observations}

\maketitle

%\begin{center}
%{\Large \bf Very early and incomplete version. Do not distribute!}
%\end{center}

\begin{abstract}
In most papers establishing consistency for learning algorithms it is assumed that the observations 
used for training are realizations of an i.i.d.~process. In this paper we go far beyond this classical framework 
by showing that support vector machines (SVMs) essentially only require that the data-generating process satisfies a 
certain law of large numbers. We then consider the learnability of SVMs for $\a$-mixing 
(not necessarily stationary) processes for both 
classification and regression, where for the latter we explicitly allow unbounded noise. 
\end{abstract}

\hspace*{15.3ex}Classification, Regression

\section{Introduction}

In recent years
Support Vector Machines (SVMs) have become one of the most widely used
algorithms for classification and regression problems.
Besides their good performance in practical applications they also enjoy 
a good theoretical justification in terms of both universal consistency (see \cite{Ste02a,Zhang04a,Ste03d,ChSt04b})
and learning rates (see \cite{ChWuYiZh04a,StSc05a,BlBoMa04a,KoBe05a,StSc04a}) if
the training samples come from an i.i.d.~process.
However, often this i.i.d.~assumption cannot be strictly justified in real-world problems. 
For example, many machine learning applications such as market prediction, system diagnosis, and
speech recognition are inherently temporal in nature, and consequently not i.i.d.~processes.
Moreover, samples are often gathered from different sources and hence it seems unlikely that 
they are identically distributed. Although SVMs have no theoretical justification in such non-i.i.d.~scenarios 
they are often applied successfully. 
One of the goals of this work is explain this success by 
establishing consistency results for 
SVMs under somewhat minimal assumptions on the data generating process. Namely, we show that 
for any data-generating process that satisfies certain laws of large numbers there exists a sequence of regularization parameters
such that the corresponding SVM is consistent.
By general negative results (see \cite{Nobel99a})
on universal consistency for stationary ergodic processes this sequence of regularization
parameters must depend on the stochastic properties of the data-generating process and cannot be adaptively chosen.
However, we show that if the process satisfies certain mixing properties such as polynomially decaying 
$\a$-mixing coefficients
(see the definitions in the following sections)
then a suitable regularization sequence can be chosen a-priori. In addition, a side-effect of our analysis 
is that it provides consistency for SVMs using Gaussian kernels even if the common compactness assumption
of the input space is violated. Consequently, our consistency results for $\a$-mixing processes generalizes 
earlier  consistency results of \cite{Ste02a,Zhang04a,Ste03d}
with respect to both the compactness assumption on $X$ and the i.i.d.~assumption on the data-generating process.

Relaxations of the independence assumption have been considered for quite a while in both the machine learning and the statistical 
literature. For example PAC-learning for stationary $\bar \b$-mixing  processes 
has been investigated in \cite{Vidyasagar02}, and more recently, consistency of regularized boosting for classification
was established for such processes. 
For a larger class of processes, namely $\a$-mixing but %\fix{Don, Clint: I cannot really improve this sentence!}
not necessarily stationary processes, consistency of kernel density estimators
was shown in \cite{Irle97a}. For bounded, stationary processes with exponentially decaying $\bar \a$-mixing coefficients
a consistent method for one-step-ahead prediction (also known as ``static autoregressive forecasting'', see \cite{GyKoKrWa02})
was presented in \cite{MoMa98a}. Moreover, for this prediction problem \cite{Meir00a} establishes consistency for a certain 
structural risk minimization approach under the assumption that the process is stationary and has polynomially decaying 
$\bar \b$-mixing rates. For further results and references we refer to \cite{GyHaSaVi89, Bosq98}.

Relaxations of the stationarity of the process are less common. In fact, to our best knowledge \cite{Irle97a}
is the only work which deals with such processes. One of the reasons for this lack of literature may be the fact that for 
non identically distributed observations there is no obvious way to define a reasonable risk functional which resembles the 
idea of ``average future error''.
On the other hand, it seems obvious that learning methods based on a modified empirical risk minimization procedure 
require at least that the process satisfies certain laws of large numbers. 
Interestingly, we will show that for processes satisfying such laws of large numbers there is always a ``limit'' distribution
which can be used to define a reasonable risk functional. 
Moreover, for many interesting classes of processes the existence of such a limit distribution turns out to be equivalent
to a law of large numbers.

The rest of this work is organized as follows:
In Section \ref{general} we will define the notions  ``laws of large numbers'' and ``limit'' distributions for stochastic processes.
We then discuss the relationship between these concepts and consider specific classes of stochastic processes that satisfy these
definitions. We then recall some basic classes of loss functions and define consistency of learning algorithms for stochastic 
processes satisfying certain laws of large numbers. Finally, we show that SVMs can be made consistent for such processes.
In Section \ref{mixing} we then recall various mixing coefficients for stochastic processes. 
These coefficient are then used to establish consistency results for SVMs with a-priori chosen regularization sequence.
Finally, the proofs of our results can be found in Section \ref{proofs}.

\section{Consistency for Processes satisfying a Law of Large Numbers}\label{general}

The aim of this section is to show that SVMs can be made consistent 
whenever the data-generating process satisfies a certain type of law of large numbers (LLNs).
To this end we first recall some notions for stochastic processes and  introduce 
these laws of large numbers in Subsection \ref{s-process}. Some examples of processes satisfying LLNs
are then presented in Subsection \ref{s-process-ex}. In Subsection \ref{loss}
we then recall some important notions for loss functions and risks. We also define consistency 
of learning algorithms for data-generating processes that satisfy a law of large numbers.
Finally, we present and discuss our consistency results for SVMs in Subsection \ref{consist-svm}.

\subsection{Law of Large Numbers for Stochastic Processes}\label{s-process}

In this subsection we mainly introduce laws of large numbers for general, not necessarily 
stationary stochastic processes. The concepts we will present seem to be
quite natural and elementary, and therefore one would expect that they have already been
introduced elsewhere. Surprisingly, however, we were not able to find any 
exposition that covers major parts of the material of this section, and thus we discuss  the following 
notions in some detail.

Let us begin with some notations.
Given a measurable space $Z$ we write  $\sL 0 Z$ for the set of all measurable functions $f:Z\to \R$, and $\sL \infty Z$ for
the set of all bounded measurable functions $f:Z\to \R$. 
Moreover,  for a set $B\subset Z$ we write $\eins_B$ for its 
indicator function, i.e.~$\eins_B:Z\to \{0,1\}$ with $\eins_B(z)=1$ if and only if $z\in B$.
Let us now assume that we also have a probability space $(\Om,\ca A,\mu)$  and a
measurable map $T:\Om \to Z$.  Then $\s(T)$ denotes the smallest $\s$-algebra on $\Om$ for which
$T$ is measurable. Moreover, $\mu_T$ denotes the $T$-image measure
of $\mu$, which is defined by $\mu_T(B):= \mu(T^{-1}(B))$, $B\subset Z$ measurable. 

Again, let $(\Om,\ca A,\mu)$ be a probability space and 
$(Z,\ca B)$ be a measurable space. 
Recall that for a  {\em stochastic process\/} 
$\ca Z:=(Z_i)_{i\geq 1}$, i.e.~a sequence of measurable maps $Z_i:\Om \to Z$, $i\geq 1$,
the map $\ca Z:\Om \to Z^\N$ defined  by $\om\mapsto (Z_i(\om))_i$
is   $(\ca A,\B^\N)$-measurable.
Consequently,  $\ca Z$ has an image measure 
$\mu_{\ca Z}$ which is given by 
$\mu_{\ca Z}(B) := \mu(\ca Z^{-1}(B))$ for all  $B\subset \ca B^\N$.

Furthermore, recall that $\ca Z$
is called {\em identically distributed\/} if $\mu_{Z_i} = \mu_{Z_j}$ for all 
$i,j\geq 1$,  and {\em stationary in the wide sense\/} if $\mu_{(Z_{i_1+i}, Z_{i_2+i})} = \mu_{(Z_{i_1}, Z_{i_2})}$
for all $i_1,i_2,i\geq 1$. Moreover, $\ca Z$ is said to be 
{\em stationary\/} if
$\mu_{(Z_{i_1+i},\dots,Z_{i_n+i})} = \mu_{(Z_{i_1},\dots,Z_{i_n})}$
for all $n,i,i_1,\dots,i_n\geq 1$.

As we will see later we are not interested in the data-generating process $\ca Z:= (Z_i)$ itself, but only in processes
of the form $g\circ \ca Z:= (g\circ Z_i)_{i\geq 1}$ for $g:Z\to Z'$ measurable. In the following we call $g\circ \ca Z$ an {\em image}
of the process $\ca Z$, and $\ca Z$ itself  a {\em hidden} process.
The following definition introduces laws of large numbers for stochastic processes by considering real-valued image processes:

\begin{definition}
Let $(\Om,\ca A,\mu)$ be a probability space,
$Z$ be a measurable space, and $\ca Z:=(Z_i)_{i\geq 1}$ be a $Z$-valued stochastic process on $\Om$.
We say that $\ca Z$ satisfies the {\em weak law of large numbers for events (WLLNE)}
if for all measurable $B\subset Z$ there exists a constant $c_B\in \R$ such that for all $\e>0$ we have
	\begin{equation}\label{wlln-events}
	\lim_{n\to \infty}\mu\biggl( \Bigl\{   \om \in \Om: \Bigl|   \frac 1 n \sum_{i=1}^n \eins_B \circ Z_i(\om) - c_B\Bigr|>\e\Bigr\}\biggr) = 0\, .
	\end{equation}
Moreover, we say that $\ca Z$ satisfies the {\em strong law of large numbers for events (SLLNE)}
if for all measurable $B\subset Z$ there exists a constant $c_B\in \R$ with 
	\begin{equation}\label{slln-events}
	\lim_{n\to \infty}\frac 1 n \sum_{i=1}^n \eins_B\circ Z_i(\om) = c_B
	\end{equation}
	for $\mu$-almost all $\om\in \Om$.
\end{definition}

It is obvious that $\ca Z$ satisfies the WLLNE  if and only if 
the sequences $(\frac 1 n \sum_{i=1}^n \eins_B\circ Z_i)$ converge in probability $\mu$ 
for all measurable $B\subset Z$.  Consequently,  the SLLNE implies the WLLNE
but in general the converse implication does not hold.
Moreover, if $\ca Z$ satisfies the WLLNE then the constants $c_B$ in (\ref{wlln-events}) must obviously satisfy 
$c_B\in [0,1]$ for all measurable $B\subset Z$.
Finally, if $\ca Z$ satisfies the WLLNE or SLLNE
then it is a trivial exercise to check that every image $g \circ \ca Z$ also satisfies the WLLNE or SLLNE, respectively. 

It is well known that i.i.d.~processes generated by $P$ satisfy the $P^\infty$-SLLNE with $c_B = P(B)$ for all measurable $B\subset Z$, but 
these processes are by far not the only ones (see Subsection \ref{s-process-ex} for some other examples).
For the following development it is instructive to observe  that for i.i.d.~processes the map $B\mapsto c_B$ defines a probability measure
on $Z$. Our next goal is to show that this remains true for general processes satisfying a WLLNE.
To this end we first consider the averages $\frac 1 n \sum_{i=1}^n \E_\mu\eins_B\circ Z_i$ of the 
probabilities of the event $B$:

\begin{definition}
Let $(\Om,\ca A,\mu)$ be a probability space,
$Z$ be a measurable space, and $\ca Z:=(Z_i)_{i\geq 1}$ be a $Z$-valued stochastic process on $\Om$.
We say that $\ca Z$
is {\em asymptotically mean stationary (AMS)}
if 
\begin{equation}\label{lln-expectations}
P (B) := \lim_{n\to \infty}\frac 1 n \sum_{i=1}^n \E_\mu \eins_B\circ Z_i
\end{equation}
exists for all  measurable $B\subset Z$.
\end{definition}

The notion ``asymptotically mean stationary'' was first introduced for dynamical systems 
by Grey and Kieffer in \cite{GrKi80a}. We are unaware of any work that introduces this notion for 
general stochastic processes, though a similar idea already appears as assumption (S1) in \cite{Irle97a}.

Using the simple formula $\eins_B\circ g= \eins_{g^{-1}(B)}$ 
it is obvious that every image $g\circ \ca Z$ of an AMS process $\ca Z$ is again AMS.
Moreover, identically distributed---and hence stationary---processes are obviously AMS.
Moreover,  for such processes we also have 
$P(B)= \mu_{Z_1}(B)$ for all measurable $B\subset Z$, and 
consequently, (\ref{lln-expectations}) defines a probability measure on $Z$.
The following lemma whose proof can be found in Section \ref{proofs}
shows that the latter observation remains true for general AMS processes.

\begin{lemma}\label{limit-distrib}
Let $(\Om,\ca A,\mu)$ be a probability space,
$Z$ be a measurable space, and $\ca Z:=(Z_i)_{i\geq 1}$ be a $Z$-valued stochastic process on $\Om$ which is
AMS.
Then $P$ defined by (\ref{lln-expectations}) is a probability measure on $Z$. We call $P$ the {\em stationary mean of $(\ca Z,\mu)$.}
\end{lemma}

It it well-known that not every stationary process satisfies a (weak, strong) law of large numbers for events.
Consequently, we see that in general AMS processes do not satisfy a law of large numbers.
However, the following theorem proved in Section \ref{proofs} shows that the converse implication is true.
In addition, it shows that the constants $c_B$ in (\ref{wlln-events})
define the stationary mean distribution:

\begin{theorem}\label{exists-limit-distrib-th}
Let $(\Om,\ca A,\mu)$ be a probability space,
$Z$ be a measurable space, and $\ca Z:=(Z_i)_{i\geq 1}$ be a $Z$-valued stochastic process on $\Om$  satisfying
the WLLNE.  
Then $\ca Z$ is AMS and the stationary mean $P$ of $(\ca Z,\mu)$ satisfies 
\begin{equation}\label{slln-events-ams}
\lim_{n\to \infty}
\mu\biggl( 
\Bigl\{   
\om \in \Om: 
\Bigl|   
\frac 1 n \sum_{i=1}^n \eins_B\circ Z_i(\om) - P(B)
\Bigr|
>\e
\Bigr\}
\biggr) = 0\, 
\end{equation}
for all  measurable $B\subset Z$ and all $\e>0$. Moreover, if $\ca Z$  satisfies the SLLNE then 
$$
\lim_{n\to 0} \frac 1 n \sum_{i=1}^n \eins_B\circ Z_i(\om) = P(B)
$$
holds for $\mu$-almost all $\om \in \Om$.
\end{theorem}

Equation (\ref{slln-events-ams}) shows that the stationary mean $P$ describes with high probability our average  observations from $\ca Z$.
Given a loss function $L$ (see Subsection \ref{loss} for definitions)
it seems therefore natural to approximate the empirical $L$-risk of a function by the corresponding
$L$-risk  defined by $P$.\footnote{For i.i.d.~observations one typically argues the other way around. 
However, for general stochastic processes the learning goal should be to minimize the future average loss. This loss 
is an empirical $L$-risk which can be approximated by the $L$-risk defined by $P$. In the training phase
of empirical risk minimizers the latter $L$-risk
is then approximated by the empirical $L$-risk of the already observed training samples.
In this way $P$ and the corresponding convergence rates in  (\ref{lln-expectations}) and  (\ref{slln-events-ams}) 
tell us how well we can generalize from the past to the future.}
However, in order to make this ansatz rigorous we have to extend (\ref{slln-events-ams}) to function classes larger than 
the set of indicator functions.
We begin with the following result that shows that a law of large numbers for events implies a corresponding law
of large numbers of bounded functions:

\begin{lemma}\label{lln-bounded}
Let $(\Om,\ca A,\mu)$ be a probability space,
$Z$ be a measurable space, and $\ca Z:=(Z_i)_{i\geq 1}$ be a $Z$-valued stochastic process on $\Om$  satisfying
the WLLNE. Furthermore, let $P$ be the asymptotic mean of $(\ca Z,\mu)$. Then for all $f\in \sL \infty Z$ we have
\begin{equation}\label{lln-bounded-hx}
\E_P f = \lim_{n\to \infty}\frac 1 n \sum_{i=1}^n  f\circ Z_i
\end{equation}
in probability $\mu$ and
\begin{equation}\label{lln-bounded-h0}
\E_P f= \lim_{n\to \infty}\frac 1 n \sum_{i=1}^n \E_\mu f\circ Z_i\, .
\end{equation}
Moreover, if $\ca Z$ actually satisfies the SLLNE then the convergence in (\ref{lln-bounded-hx}) holds $\mu$-almost surely.
\end{lemma}

For classification problems we usually can restrict our considerations to bounded functions, and hence 
Lemma \ref{lln-bounded} is all that we need. However, for regression problems with unbounded noise we 
have to consider {\em integrable\/} functions, instead. 
The following definition serves this purpose:

\begin{definition}\label{lln-l1-def}
Let $(\Om,\ca A,\mu)$ be a probability space,
$Z$ be a measurable space, and $\ca Z:=(Z_i)_{i\geq 1}$ be a   $Z$-valued stochastic process on $\Om$.
Assume that $\ca Z$ is AMS and  let $P$ be the asymptotic mean of $(\ca Z,\mu)$. %, and $f\in \Lx 1 P$. % be a measurable function with
%$f\in \Lx 1 P$ and $f\circ Z_i \in \Lx 1 \mu$ for all $i\geq 1$.
We say that $\ca Z$ satisfies the {\em weak law of large numbers (WLLN)}
if for all $f\in \Lx 1 P$ and all $\e>0$ we have 
	\begin{equation}\label{wlln-l1}
	\lim_{n\to \infty}\mu\biggl( \Bigl\{   \om \in \Om: \Bigl|   \frac 1 n \sum_{i=1}^n f \circ Z_i(\om) - \E_P f\Bigr|>\e\Bigr\}\biggr) = 0\, .
	\end{equation}
Moreover, we say that $\ca Z$ satisfies the {\em strong law of large numbers  (SLLN)} if for all $f\in \Lx 1 P$ we have 
	\begin{equation}\label{slln-l1}
	\lim_{n\to \infty}\frac 1 n \sum_{i=1}^n f \circ Z_i(\om) = \E_P f
	\end{equation}
	for $\mu$-almost all $\om\in \Om$.
\end{definition}

%Finally, assume that  $\ca Z$ satisfies the WLLN and $f:Z\to \R$ is a measurable function such that 
% $\{f\circ Z_i: i\geq 1\}$ is uniformly integrable. Using Fatou's lemma it is not hard to see that $f\in \Lx 1 P$, and consequently 
% (\ref{wlln-l1})  holds. Moreover, the uniform integrability shows that the convergence in (\ref{wlln-l1}) also holds with respect to the 
% $\Lx 1 \mu$-norm and that in addition, the convergence (\ref{lln-bounded-h0}) is true.

\subsection{Examples of Processes Satisfying a Law of Large Numbers}\label{s-process-ex}

In this subsection we recall several examples of stochastic processes satisfying a law of large numbers.
In particular, we consider independent processes, dynamical systems, and Markov chains.

\subsubsection{Uncorrelated and independent processes}

Recall that two real-valued random variables $\xi$ and $\eta$ are called {\em uncorrelated\/}
if they satisfy $\E \xi \eta = \E \xi \, \E\eta$. The following proposition proved in Section \ref{proofs} 
shows that AMS, mutually uncorrelated processes 
satisfy a WLLNE:

\begin{proposition}\label{uncor-wllne}
Let $(\Om,\ca A,\mu)$ be a probability space,
$Z$ be a measurable space, and $\ca Z:=(Z_i)_{i\geq 1}$ be a $Z$-valued stochastic process on $\Om$. 
Assume that the random variables $\eins_B\circ Z_i$ and $\eins_B\circ Z_j$
are uncorrelated for all measurable $B\subset Z$ and all $i,j\geq 1$ with $i\neq j$. Then 
the following statements are equivalent:
\begin{enumerate}
\item $\ca Z$ is AMS. 
\item $\ca Z$ satisfies the WLLNE.
\end{enumerate}
\end{proposition}

Considering the proof of the above proposition it is immediately clear that the proposition 
remains true if the process is not uncorrelated but only satisfies
\begin{equation}\label{uncor-limit-h1}
\lim_{n\to \infty}  \E_\mu \Bigl( \frac 1 {n^2} \sum_{i=1}^n \bigl(\eins_B\circ Z_i - \E_\mu \eins_B\circ Z_i  \bigr)  \Bigr)^2 = 0
\end{equation}
for all measurable $B\subset Z$. Processes satisfying such a weaker assumption are introduced and discussed in Subsection
\ref{mixing-coeff}.

It is obvious that Proposition \ref{uncor-wllne} holds
for processes for which the image processes $(\eins_B\circ Z_i)_{i\geq 1}$ are independent. However, 
by applying \cite[Theorem 2.7.1]{Revesz68} we have the following stronger result:

\begin{proposition}\label{indep-sllne}
Let $(\Om,\ca A,\mu)$ be a probability space,
$Z$ be a measurable space, and $\ca Z:=(Z_i)_{i\geq 1}$ be a $Z$-valued stochastic process on $\Om$. 
Assume that  $\eins_B\circ Z_1, \eins_B\circ Z_2, \dots$ are independent for all fixed measurable $B\subset Z$.
Then 
the following statements are equivalent:
\begin{enumerate}
\item $\ca Z$ is AMS. 
\item $\ca Z$ satisfies the SLLNE.
\end{enumerate}
\end{proposition}

Note that the independence assumption in Theorem \ref{indep-sllne} is weaker than 
assuming that the process is independent.

By Kolmogorov's well-known strong law of large numbers it is obvious that every process $\ca Z$ whose $\R$-valued  images
$g\circ \ca Z$
 are i.i.d.~processes satisfies 
a SLLN. Moreover, a result by Etemadi \cite{Etemadi81a} shows that the independence assumption can be relaxed to pairwise 
independence.
Finally, the following result whose proof can again be found in Section \ref{proofs}
generalizes Kolmogorov's law of large numbers   to a certain type of martingale:

\begin{proposition}\label{marting-slln}
Let $(\Om,\ca A,\mu)$ be a probability space,
$Z$ be a measurable space, and $\ca Z:=(Z_i)_{i\geq 1}$ be a $Z$-valued stochastic process on $\Om$. 
Assume that for all $f\in \Lx 1 {\mu_{Z_1}}$ and $\ca F_n:= \s(f\circ Z_i: i\geq n)$, $n\geq 1$, we have 
$\bigcap_{i\geq 1} \ca F_i  = \{\emptyset, \Om\}$ and
\begin{equation}\label{marting-slln-h1}
\E \biggl(\frac 1 n \sum_{i=1}^n f\circ Z_i\, \Bigl| \, \ca F_{n+1}   \biggr)  =  \frac1 {n+1} \sum_{i=1}^{n+1} f\circ Z_i \, .
\end{equation}
Then $\ca Z$ satisfies the SLLN and $\mu_{Z_1}$ is the asymptotic mean of $(\ca Z,\mu)$.
\end{proposition}

\subsubsection{Ergodic processes}

In this section we recall the basic notions and results for dynamical systems. 
To this end let $Z$ be a measurable space and $S:Z^\N\to Z^\N$ be the shift operator defined by 
$(z_i)\mapsto (z_{i+1})$. A set $B\subset Z^\N$ is called {\em invariant\/} if $S^{-1}(B) = B$.
Moreover, let $(\Om,\ca A,\mu)$ be a probability space
and $\ca Z:=(Z_i)_{i\geq 1}$ be a $Z$-valued stochastic process on $\Om$.
Then $\ca Z$ is called {\em ergodic\/} if we have $\mu_{\ca Z}(B)\in \{0,1\}$ for all measurable invariant subsets 
$B\subset Z^\N$.
It is not hard to see that every image of an ergodic process is again an ergodic process. 

In the following we are mainly interested 
in stationary ergodic processes. To this end 
let us now assume that  $(Z,\ca B, \mu)$ is a probability  space and $T:Z\to Z$ is a measurable map.
Then the stochastic process   $\ca Z:= (T^{i-1})_{i\geq 1}$ is called a {\em dynamical system}, and it is  
 called an {\em invariant dynamical system\/}  if  the $T$-image $\mu_T$ of $\mu$
satisfies
$\mu=\mu_T$. 
Recall that 
an invariant dynamical system $\ca Z:= (T^{i-1})_{i\geq 1}$ on a probability  space $(Z,\ca B, \mu)$    is ergodic if and only if 
$\mu$ satisfies $\mu(B)\in \{0,1\}$ for all measurable $B\subset Z$ with $T^{-1}(B)=B$.
Moreover, recall that every stationary process is the image of a hidden invariant dynamical system. 
Conversely, every invariant dynamical system is stationary and hence AMS.
In addition recall that Birkhoff's theorem (see e.g.~\cite[p.~82ff]{BrSt02}):
% characterizes invariant 
%dynamical systems satisfying a strong law of large numbers in terms 
%of ergodicity. Namely we have:

\begin{theorem}\label{ergodic-char}
Let $\ca Z:= (T^{i-1})_{i\geq 1}$ be an invariant dynamical system on a probability  space $(Z,\ca B, \mu)$.
Then
the following statements are equivalent:
\begin{enumerate}
\item $\ca Z$ satisfies the SLLNE.
\item $\ca Z$ satisfies the SLLN.
\item $\ca Z$ is ergodic.
\end{enumerate}
\end{theorem}

With the help of the above theorem one can show (see e.g.~\cite[p.~26f]{Krengel85})
that every stationary ergodic process $\ca Z$ satisfies the SLLN.
Moreover, by a theorem by Gray and Kieffer
(see e.g.~\cite [p.~33]{Krengel85}) we know that a dynamical system $\ca Z:= (T^{i-1})_{i\geq 1}$ 
is AMS if and only if $\lim_{n\to \infty}\frac 1n\sum_{i=1}^n f\circ T^{-1}$ exists $\mu$-almost surely for all $f\in \sL \infty Z$.
Note that Birkhoff's theorem shows
that the corresponding limit is a {\em constant\/} function if and only if the dynamical system is ergodic.
Finally, it is interesting to note that for  stationary, ergodic processes the limit relation (\ref{uncor-limit-h1}) holds
(see e.g.~\cite[Thm.~2.19, p.~61]{Bradley05a}).

Let us now recall a  notion related to ergodicity. To this end let
$(Z,\ca B,\mu)$ be a probability  space
and $\ca Z:= (T^{i-1})_{i\geq 1}$ be an  invariant dynamical system on $Z$. 
%Then it is well-known (see e.g.~\cite[p.~52]{Rudolph90}) that $\ca Z$ is 
%ergodic if and only if 
%$$
%\lim_{n\to \infty} \frac 1 n \sum_{i=0}^{n-1}\, \mu\bigl(T^{-i}(A) \cap B\bigr) \Eq \mu(A)\mu(B)  \, , \qquad \qquad A,B\in \ca B.
%$$
%There are several ways to consider sharper limit conditions. Let us recall the two most common ones: 
Then
$\ca Z$ is said to be {\em weakly  mixing\/} if 
$$
\lim_{n\to \infty} \frac 1 n \sum_{i=0}^{n-1} \,\,\,\Bigr| \mu\bigl(T^{-i}(A) \cap B\bigr) - \mu(A)\mu(B) \Bigl| = 0\, , \qquad \qquad A,B\in \ca B.
$$
%and {\em  mixing\/} if 
%$$
%\lim_{n\to \infty} \,\, \Bigr| \mu\bigl(T^{-n}(A) \cap B\bigr) - \mu(A)\mu(B) \Bigl| = 0\, , \qquad \qquad A,B\in \ca B.
%$$
It is well-known that  weak mixing implies ergodicity, and that 
that the converse implication does not hold in general (see e.g.~\cite[p.~41ff]{Petersen89}).
Moreover, one can also introduce mixing conditions for general stationary ergodic processes. For example, 
if $(\Om,\ca A,\mu)$ is a probability space,
$Z$ is a measurable space, and $\ca Z:=(Z_i)_{i\geq 1}$ is a $Z$-valued stochastic process on $\Om$, then $\ca Z$ is called
mixing if 
\begin{equation}\label{mixing-general}
\lim_{n\to \infty} \mu_{\ca Z} \bigl(S^{-n}(A) \cap B\bigr) = \mu_{\ca Z}(A)\mu_{\ca Z}(B)
\end{equation}
holds for all measurable $A,B\subset Z^\N$. One can show (see e.g.~\cite[Prop.~2.8, p.~50]{Bradley05a})
that for invariant dynamical systems this definition coincides with the above mixing definition.
Moreover, recall that  i.i.d.~processes are invariant and weakly mixing (see \cite[p.~58]{Petersen89}).

%\begin{example}\label{shift-mix}
%Given a  probability space $(\Om,\ca A,\mu)$ we define $(Z,\ca B,\nu):= (\Om^\N,\ca A^\N,\mu^\N)$. Furthermore, let $S:\Om^\N\to \Om^\N$
%be the shift operator defined by $(\om_i)_{i\geq 1}\mapsto (\om_{i+1})_{i\geq 1}$. 
%Then the dynamical system  
%$(S^{i-1})_{i\geq 1}$ on $Z$ is obviously invariant and \cite[p.~58]{Petersen89} shows it is also 
%mixing. 
%\end{example}

The weak mixing  is important because it allows us to establish the ergodicity of products of dynamical systems.
This leads to our last example:

%More precisely the following theorem (see e.g.~\cite[p.~65]{Petersen89}) holds:

%\begin{theorem}\label{prod-ergod}
%Let $(Z_1,\ca B_1,\mu_1)$ and $(Z_2,\ca B_2,\mu_2)$ be probability spaces and
%$\ca Z_j:= (T^{i-1}_j)_{i\geq 1}$ be invariant dynamical systems on $Z_j$, $j=1,2$. 
% If $\ca Z_1$ is weakly  mixing and $\ca Z_2$ is  ergodic then $\ca Z_1\times \ca Z_2:= (T_1^{i-1}\times T_2^{-1})_{i\geq 1}$
%is an  invariant ergodic dynamical system on $(Z_1\times Z_2,\ca B_1\otimes \ca B_2,\mu_1\otimes \mu_2)$.
%\end{theorem}
%With this theorem we obtain the following corollary which is important when learning dynamical systems from noisy observations:

\begin{proposition}\label{noisy-DS}
Let $\mu$ be a probability measure on $\Rn$ and $\ca Z$ be an invariant ergodic dynamical system on $(\Rn,\mu)$.
Furthermore, let $(\Om,\ca A,\nu)$ be a probability space and 
 $\ca E$ be an i.i.d.~sequence of 
random variables $\e_i:\Om\to\Rn$.
Then the process $\ca Z + \ca E$ defined on $(\R^n\times \Om, \mu\otimes \nu) $ satisfies the SLLN.
\end{proposition}

\subsubsection{Markov chains}

In this subsection we briefly discuss a law of large numbers for Markov chains. 
To this end let us fix a probability  space $(Z,\ca B,\nu)$. Furthermore, let $p:\ca B\times Z\to [0,1]$
be a stochastic transition function, i.e.~a Markov kernel.
%, i.e.~a map that satisfies the following conditions:
%\begin{enumerate}
%\item $p(\,.\,,z)$ is a probability measure on $\ca B$ for all $z\in Z$.
%\item $p(B,\,.\,):Z\to [0,1]$ is measurable for all $B\in \ca B$.
%\end{enumerate}
Let us define a probability measure $P$ on $(Z^\N,\ca B^\N)$ by
\begin{equation}\label{markov-im}
P(B_1\times \dots \times B_n) := \int \eins_{B_1\times\dots\times B_n}(z_1,\dots,z_n) p(dz_n,z_{n-1})\dots p(dz_2,z_1)\nu(dz_1)\, ,
\end{equation}
where $n$ runs over all integers and $B_1,\dots,B_n$ run over all measurable subsets of $Z$.
A $Z$-valued stochastic process $\ca Z$ defined on a probability space $(\Om,\ca A,\mu)$ 
 is called {\em homogeneous\footnote{Since we only deal with homogeneous Markov chains we often omit the adjective ``homogeneous''.} 
 Markov chain\/} with 
transition function $p$ and initial distribution $\nu$ if it satisfies $\mu_{\ca Z} = P$, where $P$
is determined by  (\ref{markov-im}). Obviously, the sequence $(\pi_i)_{i\geq 1}$ of 
coordinate projections $\pi_i:Z^\N\to Z$, $(z_j)\mapsto z_i$ is a canonical model of such a Markov chain if $Z^\N$ is equipped with the distribution
$P$. Moreover, if the homogeneous Markov chain is stationary then $\nu$ satisfies $\mu_{Z_i}= \nu$ for all $i\geq 1$.

The transition function describes the probability of $Z_{n+1}$ given the state of the process at time $n$.
For larger steps ahead one can iteratively compute the corresponding transition probabilities by
\begin{eqnarray*}
p^{(1)}(B,z) & = & p(B,z)\\
p^{(n+1)} (B,z) & = & \int p^n(B,z') p(dz',z) \, .
\end{eqnarray*}
Let us now assume that there exists a finite measure $Q$ on $\ca B$ with $Q(Z)>0$, an integer $n\geq 1$, and a real number $\e>0$ such that 
for all measurable $B\subset Z$ we have
\begin{equation}\label{doeblin}
Q(B)\leq \e \qquad \qquad \implies \qquad \qquad p^{(n)}(B,z)\leq 1-\e \quad\mbox{ for all } z\in Z\, .
\end{equation}
This assumption taken from \cite[p.~192]{Doob53} is often called the ``Doeblin condition'' (see e.g.~\cite[p.~197]{Doob53} or \cite[p.~156]{BhWa01a}). 
If $Z$ is a finite set,
then (\ref{doeblin}) is automatically satisfied (see e.g.~\cite[p.~192]{Doob53}). Moreover, if $Z\subset \R^d$ is a set of finite Lebesgue measure
and the distributions $p(\,.\,z)$, $z\in Z$ are absolutely continuous with uniformly bounded transition densities then (\ref{doeblin})
also holds (see e.g.~\cite[p.~193]{Doob53}). For some similar conditions we finally refer to \cite{BhWa01a} and the references therein).

Now, the following theorem which can be found in \cite[p.~219]{Doob53} gives a simple condition ensuring a SLLN for Markov chains:

\begin{theorem}
Let $(Z,\ca B,\nu)$ be a probability space, $p:\ca B\times Z\to [0,1]$
be a stochastic transition function and $\ca Z=(Z_i)_{i\geq 1}$ be a stationary homogeneous Markov chain with 
transition function $p$ and initial distribution $\nu$. If $\ca Z$ satisfies (\ref{doeblin}) then 
$\ca Z$ satisfies the SLLN.
\end{theorem}

The above theorem can be generalized to non-homogeneous, not identically distributed Markov chains.
Since these generalizations are out of the scope of the paper  
we refer to \cite[p.~129-135]{Revesz68} for details.
Finally, we would also like to mention without explaining the details that if $Z$ is a countable set then an irreducible, positive recurrent, 
homogeneous Markov chain satisfies the SLLNE (see e.g.~\cite[Thm.~1.10.2]{Norris97}).

\subsection{Loss functions, Risks, and Consistency}\label{loss}

In this section we recall some basic notions for loss functions and their associated risks.
We then introduce consistency notions for learning algorithms for stochastic processes satisfying a law of large numbers.

In the following $X$ is always a measurable space if not mentioned otherwise and $Y\subset \R$ is always a closed subset. 
Moreover, metric spaces are always equipped with the Borel $\s$-algebra, and products of measurable spaces are always equipped with the 
corresponding product  $\s$-algebra.  
Finally, $\Lx p\mu$
stands for the standard space of $p$-integrable functions with respect to the measure $\mu$ on $X$.

\begin{definition}\label{def:loss-and-risk}
A function
$L:X\times Y\times \R \to [0,\infty]$ is called a {\em loss function} if it is measurable.
In this case $L$ is called:
\begin{enumerate}
\item {\em  convex} if $L(x,y,\,.\,):\R\to [0,\infty]$ is convex for all $x\in X$, $y\in Y$.
\item {\em continuous} if $L(x,y,\,.\,):\R\to [0,\infty]$ is continuous for all $x\in X$, $y\in Y$.
\end{enumerate}
Moreover, for a probability measure $P$ on $X\times Y$ and an $f\in \sL 0 {X}$ the 
{\em $L$-risk} of $f$
is defined by
$$
\RP L  f
:=
\Int_{X\times Y} L\bigl(x,y,f(x)\bigr) \, dP(x,y)
=
\Int_{X}\Int_Y L\bigl(x,y,f(x)\bigr) \, dP(y|x)\,dP_X(x).
$$
Finally, the {\em Bayes $L$-risk} is $\RPB L := \inf\{ \RP L f: f\in \sL 0X\}$.
%Finally, an $f_{L,P}^{*}\in \sL 0 {X}$ with $\RP L {f_{L,P}^{*}}= \RPB L$ is called a {\em Bayes  function}.
\end{definition}

Note that the integral defining the $L$-risk always exists since $L$ is non-negative and measurable.
In addition it is obvious that the risk of a convex loss is  convex  on $\sL 0 X$.
However, in general the risk of a continuous loss is not continuous. In order to ensure this continuity 
and several other, more sophisticated properties we need the following 
definition:

\begin{definition}\label{loss:nemitski-def}
We call a loss function $L:X\times Y\times \R\to [0,\infty]$ %be a loss function. % and $\P$ be a distribution on $X\times Y$.
a {\em Nemitski loss function} if %with respect to $\P$ if 
there exist a measurable function $b:X\times Y\to [0,\infty)$ and an increasing function $h:[0,\infty)\to [0,\infty)$
with
\begin{equation}\label{infinite:bound-loss}
L(x,y,t) \leq b(x,y) + h\bigl(|t|\bigr)\, , \qquad \qquad (x,y,t)\in X\times Y\times \R.
\end{equation}
Furthermore, we say that
$L$ is a {\em Nemitski loss of order $p\in(0,\infty)$}, if 
there exists a constant $c>0$ with  $h(t) = c\,t^p$ for  all $t\geq 0$.
%there exists a $c>0$ with
%$$
%L(x,y,t) \leq b(x,y) + c\, |t|^{p}\, , \qquad \qquad (x,y,t)\in X\times Y\times \R.
%$$
Finally, if $P$ is a distribution on $X\times Y$ with $b\in \Lx 1 P$ we say that $L$ is a {\em $P$-integrable Nemitski loss.}
\end{definition}

Note that $P$-integrable Nemitski loss functions $L$  satisfy 
$\RP L f<\infty$ for all $f\in L_{\infty}(P_{X})$, and consequently we also have $\RP L 0<\infty$ and
$\RPB L <\infty$.

For our further investigations we also need the following additional properties which are satisfied 
by basically all commonly used loss functions:

\begin{definition}
Let $L:X\times Y\times \R \to [0,\infty)$ be a loss function. We say that $L$ is:
\begin{enumerate}
\item {\em locally bounded} if for all bounded $A\subset \R$ the restriction $L_{|X\times Y\times A}$ of $L$ 
	%onto
	%$X\times Y\times A$ 
	is a bounded function.
\item {\em locally Lipschitz continuous}  if for all $a>0$ we have
	\begin{equation}\label{loss:lipschitz-loss-def}
	|L|_{a,1} := \sup_{\substack{t,t'\in [-a,a]\\t\neq t'}}\,\,\sup_{\substack{x\in X\\ y\in Y}} \frac{ \bigl| L(x,y, t) - L(x,y,t')\bigr| }{ |t-t'|} \,<\, \infty\, .
	\end{equation}
\item {\em Lip\-schitz continuous} if we have $|L|_{1}:=\sup_{a>0}|L|_{a,1}<\infty$.
\end{enumerate}
\end{definition}

Note that if $Y\subset \R$ is a {\em finite} subset and $L:Y\times \R\to [0,\infty)$
is a convex loss function then $L$ is a locally Lipschitz continuous loss function.
%since every finite convex function is locally Lip\-schitz continuous.
Moreover, a locally Lip\-schitz continuous loss function $L$ is a Nemitski loss since (\ref{loss:lipschitz-loss-def})
 yields
\begin{equation}\label{loc-lip-imply-nemits}
L(x,y,t) \leq L(x,y,0) + |L|_{|t|,1}|t|\, , \qquad \qquad (x,y,t)\in X\times Y\times \R.
\end{equation}
In particular, a locally Lipschitz continuous loss $L$ is a $P$-integrable Nemitski loss if and only if 
$\RP L 0<\infty$. Moreover, if $L$ is 
Lipschitz continuous then $L$ is a Nemitski loss of order $1$.

The following examples recall that 
(locally) Lipschitz continuous losses are often used in learning algorithms for 
classification and regression problems:

\begin{example}
A loss  $L:Y\times \R\to [0,\infty)$ of the form $L(y,t)= \p(yt)$ for a suitable function $\p:\R\to \R$ and all $y\in Y:= \{-1,1\}$ and  $t\in \R$,
is called  {\em margin-based\/}. 
Recall that margin-based losses such as the (squared) hinge loss, the AdaBoost loss, the logistic loss and the least squares loss
are used in many classification algorithms.
Obviously, $L$ is convex, continuous, or (locally)
Lipschitz continuous if and only if $\p$ is. In addition,  convexity of $L$ implies  local Lipschitz continuity of $L$.
Moreover, 
$L$ is always a $P$-integrable Nemitski loss %for all measurable spaces $X$ and all distributions $P$ on $X\times Y$ 
since
we have
\begin{equation}\label{m-based-n}
L(y,t) \leq \max\{\p(-t),\p(t)\}\, %,\qquad \qquad  y\in Y,\, t\in \R .
\end{equation}
for all $y\in Y$ and all $t\in \R$. In particular, this estimate shows that every convex margin-based loss is locally bounded.
Moreover,  from (\ref{m-based-n}) we can easily derive a characterization for $L$ being a $P$-integrable Nemitski loss of order $p$.
\end{example}

\begin{example}\label{loss:ab-dist-def-ex}
A loss  $L:Y\times \R\to [0,\infty)$ of the form $L(y,t)= \psi(y-t)$  for a suitable function $\psi:\R\to \R$ and all $y\in Y:= \R$ and  $t\in \R$,
is called  {\em distance-based}.
Distance-based losses such as the least squares loss,  Huber's insensitive loss, the logistic loss, or the 
$\epsilon$-insensitive loss 
are usually used for regression.
It is easy to see that  $L$ is convex, continuous, or 
Lipschitz continuous if and only if $\psi$ is. Let us say that 
 $L$ is  of {\em upper growth $p\in [1,\infty)$} if there is a  $c>0$ with
$$
\psi(r) \leq c\, \bigl(|r|^{p} + 1 \bigr)\, , \qquad \qquad r\in \R.
$$
Analogously, $L$ is said to be of {\em lower growth $p\in [1,\infty)$} if there is a  $c>0$ with
$$
\psi(r) \geq c\, \bigl(|r|^{p} - 1 \bigr)\, , \qquad \qquad r\in \R.
$$
Recall that most of the commonly used distance-based loss functions including the above examples are of the same upper and lower growth type. 
Then it is obvious that  $L$ is of upper growth type 1 if it is Lipschitz continuous, and if $L$ is convex the converse implication
also holds. 
Moreover, non-trivial convex $L$ are always of lower growth type $1$. 
In addition, a distance-based loss function of upper growth type $p\in [1,\infty)$ is a Nemitski loss of order 
$p$, and  if the distribution $P$  satisfies the moment condition 
\begin{equation}\label{moment-def}
|P|_p := \bigl(\E_{(x,y)\sim P} |y|^p\bigr)^{1/p} := \biggl(  \int_{X\times \R} |y|^p \, dP(x,y)                 \biggr)^{1/p} \, <\, \infty
\end{equation} 
it is  also $P$-integrable. 
\end{example}

If our observations are realizations of 
a sequence $\ca Z$ of random variables $(X_i,Y_i):\Om\to X\times Y$  satisfying a law of large numbers then the following lemma 
proved in Section \ref{proofs}
shows that 
the risk with respect to the asymptotic mean distribution $P$ actually describes the average future loss.

\begin{lemma}\label{risk-confirm}
Let $(\Om,\ca A,\mu)$ be a probability space,
$X$ be a measurable space, 
 $Y\subset \R$ be a closed subset, and $\ca Z:=((X_i,Y_i))_{i\geq 1}$ be a $X\times Y$-valued stochastic process on $\Om$ satisfying
the WLLNE. Furthermore, let $P$ be the asymptotic mean of $(\ca Z,\mu)$ and $L:X\times Y\times \R\to [0,\infty)$
be a loss function.  If $L$ is locally bounded then  
	for all $f\in \sL\infty X$ and all $n_0\geq 0$ we have
	\begin{equation}\label{risk-confirm-h1}
	\RP L f = \lim_{n\to \infty} \frac 1 {n-n_0} \sum_{i=n_0+1}^n L\bigl(X_i,Y_i,f(X_i)  \bigr)\, , 
	\end{equation}
	where the limit is with respect to the convergence in probability $\mu$. Moreover, if $\ca Z$ actually satisfies the 
	SLLNE then (\ref{risk-confirm-h1}) holds $\mu$-almost surely. Finally, the same conclusions hold if $L$ is 
	a $P$-integrable Nemitski loss and $\ca Z$ satisfies the WLLN or  SLLN.  
\end{lemma}

With the help of the above lemma we can now introduce some reasonable 
concepts describing the asymptotic learning ability of learning algorithms.
To this end recall that a method $\ca L$ that provides to every {\em training set\/} $T:=((x_1,y_1),\dots,(x_n,y_n))\in (X\times Y)^n$
a (measurable) function $f_T:X\to \R$ is called a {\em learning method}. 
The following definition introduces an asymptotic way to describe whether a learning method can learn from samples:

\begin{definition}
Let $(\Om,\ca A,\mu)$ be a probability space,
$X$ be a measurable space, $Y\subset \R$ be a closed subset,
and $\ca Z:=((X_i,Y_i))_{i\geq 1}$ be a $X\times Y$-valued stochastic process on $\Om$ satisfying
the WLLNE. Furthermore, let $P$ be the asymptotic mean of $(\ca Z,\mu)$ and $L:X\times Y\times \R\to [0,\infty)$
be a loss function. We say that a learning method $\ca L$  is 
{\em $L$-consistent for $\ca Z$} if 
\begin{equation}\label{consist-h1}
\lim_{n\to \infty} \RP L {f_{T_n}}   =\RPB L
\end{equation}
holds in probability $\mu$, 
where $T_n:= ((X_1,Y_1),\dots,(X_n,Y_n))$
and $\RPB L$ is the Bayes risk defined in Definition \ref{def:loss-and-risk}. 
Moreover, we say that $\ca L$ is {\em strongly $L$-consistent for $\ca Z$}
if (\ref{consist-h1}) holds $\mu$-almost surely.
\end{definition}

\subsection{Consistency of SVMs}\label{consist-svm}

In this subsection  we present some results showing that 
support vector machines (SVMs) can learn whenever the data-generating process satisfies a law of large numbers.

Let us begin by recalling the definition of SVMs.
To this end 
let $L:X\times Y\times \R\to [0,\infty)$ be a convex loss function and $H$ be a 
reproducing kernel Hilbert space (RKHS) over $X$ (see e.g.~\cite{Aronszajn50a}).
Then for all $\lb>0$ and all observations $T:= ((x_1,y_1),\dots,(x_n,y_n))\in X\times Y$ 
there exists exactly one element $\fT\in H$ with 
\begin{equation}\label{svm-def}
\fT \in \arg\min_{f\in H} \lb \snorm f_H^2 + \frac 1 n \sum_{i=1}^n L\bigl(x_i,y_i,f(x_i)\bigr)\, .
\end{equation}
Given a  null-sequence $(\lb_n)$ of strictly positive real numbers we call the learning method 
which provides to every training set $T\in (X\times Y)^n$ the decision function $f_{T,\lb_n}$
an
{\em $(\lb_n)$-SVM\/} based on $H$ and $L$. For more information on SVMs we refer to 
\cite{CrST00,ScSm02}.

Moreover, given a distribution $P$ on $X\times Y$ we say that the RKHS $H$ is {\em $(L,P)$-rich\/} if we have 
$$
\RPxB LH := \inf_{f\in H} \RP L f = \RPB L\, ,
$$
i.e.~if the Bayes risk can  be approximated by functions from $H$.
Note that the condition $\RPxB LH = \RPB L$ is satisfied (see \cite{StHuSc06a}) whenever, the kernel of $H$ is 
universal in the sense of \cite{Ste01a}, i.e.~$X$ is a compact metric space and $H$ is dense in the space $C(X)$
of continuous functions. Less restrictive assumptions on $H$ and $X$ have been recently found in \cite{StHuSc06a}.
In particular, it was shown in \cite{StHuSc06a} that the RKHSs $H_\s$, $\s>0$, of the Gaussian RBF kernels 
$$
k_\s(x,x') \Defi \exp\bigl( -\s^2 \snorm{x-x'}_2^2 \bigr)\, , \qquad \qquad x,x'\in \Rn
$$
are  $(L,P)$-rich for all distributions $P$ on $\Rn\times Y$ and all continuous, $P$-integrable Nemitski losses $L$
of order $p\in [1,\infty)$. Finally, one can also find some necessary and sufficient conditions for $(L,P)$-richness
on countable spaces $X$ in \cite{StHuSc06a}.

In order to present our first main result let us recall that a Polish space is separable topological 
space with a countable dense subset whose topology can be described by a complete metric.
It is well known that e.g.~closed and open subset of $\R^d$ and compact metric spaces are Polish.
Now our first theorem shows that for every 
process satisfying a law of large numbers for events 
there exists an SVM which is consistent for this process:

\begin{theorem}\label{asymp-learn}
Let $X$ be a Polish space, $Y\subset \R$ be a closed subset
and  $L:X\times Y\times \R\to [0,\infty)$
be a convex, locally Lipschitz continuous, and locally bounded loss function.
Moreover, let 
$(\Om,\ca A,\mu)$ be a probability space, 
$\ca Z:=((X_i,Y_i))_{i\geq 1}$ be an $X\times Y$-valued stochastic process on $\Om$ satisfying
the WLLNE, and  $P$ be the asymptotic mean of $(\ca Z,\mu)$.
Finally, let $H$ be an $(L,P)$-rich RKHS over $X$ with continuous kernel.
Then there exists a null-sequence $(\lb_n)$ of strictly positive real numbers such that the 
$(\lb_n)$-SVM based on $H$ and $L$ is $L$-consistent for $\ca Z$. \\
In addition, if 
$\ca Z$ satisfies the SLLNE then $(\lb_n)$ can be chosen such that the $(\lb_n)$-SVM is strongly $L$-consistent for $\ca Z$.
\end{theorem}

The next theorem establishes a similar result for distance-based loss functions (see Example \ref{loss:ab-dist-def-ex})
which, in general,  are not locally bounded.

\begin{theorem}\label{asymp-learn2}
Let $X$ be a Polish space, $Y\subset \R$ be a closed subset
and  $L: Y\times \R\to [0,\infty)$
be a convex, distance-based loss function of upper growth-type $p\in [1,\infty)$.
Moreover, let 
$(\Om,\ca A,\mu)$ be a probability space, 
$\ca Z:=((X_i,Y_i))_{i\geq 1}$ be an $X\times Y$-valued stochastic process on $\Om$ satisfying
the WLLN, and  $P$ be the asymptotic mean of $(\ca Z,\mu)$.
We assume $|P|_p<\infty$.
Finally, let $H$ be the $(L,P)$-rich RKHS of a continuous kernel on $X$.
Then there exists a null-sequence $(\lb_n)$ of strictly positive real numbers such that the 
$(\lb_n)$-SVM based on $H$ and $L$ is $L$-consistent for $\ca Z$. \\
In addition, if 
$\ca Z$ satisfies the SLLN then $(\lb_n)$ can be chosen such that the $(\lb_n)$-SVM is strongly $L$-consistent for $\ca Z$.
\end{theorem}

The techniques used in the proofs of Theorem \ref{asymp-learn} and \ref{asymp-learn2} are based on a (hidden) skeleton argument 
in the proof of Lemma \ref{lln-Hilbert}. A more general though standard skeleton argument can be used to derive 
results similar to Theorem \ref{asymp-learn} and \ref{asymp-learn2} for other empirical risk minimization methods using 
hypothesis sets with
reasonably controllable complexity. Due to space constraints we omit the details.

Let us now assume for a moment that $X$ is a subset of $\R^d$, $L$ is a loss function in the sense of 
either Theorem \ref{asymp-learn} or \ref{asymp-learn2}, and  $H$ is the RKHS of a Gaussian RBF kernel.
Then the above theorems together with the richness results from \cite{StHuSc06a} 
show that for all data-generating processes $\ca Z$ satisfying a law of large numbers there exist  suitable 
regularization sequences $(\lb_n)$ that allows us to build a consistent SVM.
However, the sequences of Theorem \ref{asymp-learn} or \ref{asymp-learn2}
depend on $\ca Z$, and consequently, it would be desirable to have 
either a universal sequence $(\lb_n)$, i.e.~a sequence that guarantees consistency for all $\ca Z$, 
or
a {\em consistent\/} method that finds suitable values for $\lb$
from the observations. Unfortunately, 
the following theorem 
due to Nobel, \cite{Nobel99a}, 
together with Birkhoff's ergodic theorem  
shows that neither of these alternatives is possible:\footnote{Recall that binary classification is the ``easiest''
non-parametric learning problem in the sense that negative results  for this learning problem can typically be  translated into negative results for almost all
learning problems defined by loss functions (cf.~p.118f in \cite{DeGyLu96} for some examples in this direction
and the proof of the below theorem in \cite{Nobel99a} for the least squares loss).}

\begin{theorem}\label{nobel}
There is no learning method which is $\lls$-consistent for all stationary ergodic processes $(X_i,Y_i)$
with values in 
$[0,1]\times [0,1]$, where $\lls$ denotes the usual least square loss $\lls(y,t):= (y-t)^2$, $y,t\in \R$.
Moreover, there is no learning method which is 
$\lclass$-consistent for all stationary ergodic processes $(X_i,Y_i)$
with values in  $[0,1]\times \{-1,1\}$, where 
$\lclass$ denotes the  classification  loss $\lclass(y,t):=\eins_{(-\infty,0]}(y\sign t)$, $y=\pm 1$, $t\in \R$.
\end{theorem}

Roughly speaking the impossibility of finding a universal sequence $(\lb_n)$ is related to the fact that 
there is no {\em uniform\/} convergence speed in the LLNs for 
general processes. More precisely, if $\ca Z:=((X_i,Y_i))_{i\geq 1}$ is a stochastic process which satisfies a law of large numbers then for all 
$\e>0$, $n\geq 1$, and all suitable functions $f:X\times Y\to \R$ there exists a $\d(\e,f,n)>0$ with 
\begin{equation}\label{lln-disco}
\mu\biggl( \Bigl\{   \om \in \Om: \Bigl|   \frac 1 n \sum_{i=1}^n f \circ (X_i,Y_i)(\om) - \E_P f\Bigr|>\e\Bigr\}\biggr) \Leq \d(\e,f,n)
\end{equation}
and $\lim_{n\to \infty}\d(\e,f,n)=0$. 
Now, the proofs of Theorem \ref{asymp-learn} and Theorem \ref{asymp-learn2} (essentially) show that 
we can determine a sequence $(\lb_n)$
whenever we know such $\d(\e,f,n)$ for all $\e>0$, $n\geq 1$, and a suitably large class of functions $f$.
However, since there exists no universal sequence $(\lb_n)$ by Theorem \ref{nobel} we consequently see that 
there exists no values $\d(\e,f,n)$ such that (\ref{lln-disco}) holds for {\em all\/} (stationary) processes satisfying a law of large numbers.

This discussion shows that in order to build consistent SVMs for interesting classes of processes
one has to find quantitative versions of laws of large numbers. For i.i.d.~processes such laws have been established in recent years 
by several authors. In the following section we will present a simple yet powerful method 
for establishing  quantitative versions of laws of large numbers
 for mixing processes.

\section{Consistency for Mixing Processes}\label{mixing}

In this section we derive consistency results for SVMs under the assumption that the data-generating process
satisfies certain mixing conditions. These mixing conditions generally quantify how much a process fails to 
be independent. In the first subsection we recall some commonly used mixing conditions. In the second
subsection we then present our consistency results and compare them with known consistency results for 
other learning algorithms.

\subsection{A Brief Introduction to Mixing Coefficients for Processes}\label{mixing-coeff}

In this subsection we recall 
some standard mixing coefficients and their basic properties
(see e.g.~\cite{Bradley05a} and \cite{Bosq98} for thorough treatment). To this end let $\Om$ be a set, 
 $\ca A$ and $\ca B$ be two $\s$-algebras on $\Om$, and $\mu$ be a probability measure on $\s(\ca A\cup \ca B)$. 
Furthermore, let $H$ be a Hilbert space and $\sL p{\ca A,\mu,H}$ be the space of all 
$\ca A$-measurable $H$-valued functions that are $p$-integrable with respect to $\mu$. Using the convention $\frac 0 0 := 0$ we define
the following mixing coefficients for the pair $(\ca A,\ca B)$:
\begin{eqnarray*}
\a(\ca A,\ca B,\mu) & :=&  \sup_{\substack{A\in \ca A\\ B\in \ca B}} \bigl|  \mu(A\cap B) - \mu(A)\mu(B)     \bigr|\\
\b(\ca A,\ca B,\mu) & := &  \frac 1 2 \sup\biggl\{ \sum_{i=1}^\infty\sum_{j=1}^\infty\bigl| \mu(A_i\cap B_j) - \mu(A_i)\mu(B_j) \bigr| :
(A_i)\subset \ca A \mbox{ and } (B_j)\subset \ca B \mbox{ partitions}\biggr\}\\
\p(\ca A,\ca B,\mu) & := & \sup_{\substack{A\in \ca A\\ B\in \ca B}} \biggl|  \frac{\mu(A\cap B) - \mu(A)\mu(B)}{\mu(A)}     \biggr|\\
\p_{\mathrm{sym}}(\ca A,\ca B,\mu) & := & \sqrt{\p(\ca A,\ca B,\mu) \cdot \p(\ca B,\ca A,\mu)}\\
R_p^H(\ca A,\ca B,\mu) & := & \sup_{\substack{f\in \sL p{\ca A,\mu,H}  \\ g\in \sL p{\ca B,\mu,H} }} \biggl|  \frac{\E_\mu\langle f,g\rangle  - \langle \E_\mu f,\E_\mu g\rangle}{\snorm f_p\, \snorm g_p}         \biggr|\, , \qquad \qquad \qquad p\in [2,\infty].
\end{eqnarray*}
It is obvious from the definitions that all mixing coefficients equal 0 if $\ca A$ and $\ca B$ are independent. 
Furthermore, besides $\p$ they are all symmetric in $\ca A$ and $\ca B$.
Moreover, we have   $\a(\ca A,\ca B,\mu)  \in[0,1/4]$ and 
$\b(\ca A,\ca B,\mu), \p(\ca A,\ca B,\mu), \p_{\mathrm{sym}}(\ca A,\ca B,\mu), R_p^H(\ca A,\ca B,\mu)\in [0,1]$.
In addition, they satisfy  the relations (see e.g.~\cite[Section 1]{Bradley05a} and the references therein):
\begin{displaymath}
\begin{array}{rcccl}
2 \a(\ca A,\ca B,\mu) & \leq & \b(\ca A,\ca B,\mu) & \leq & \p(\ca A,\ca B,\mu) \\
4 \a(\ca A,\ca B,\mu) & \leq & R_p^\R(\ca A,\ca B,\mu) & \leq & 2\p_{\mathrm{sym}}(\ca A,\ca B,\mu)\, , \qquad \qquad p\in [2,\infty].
\end{array}
\end{displaymath}
Moreover, the coefficients $R_p^H(\ca A,\ca B,\mu)$ are essentially equivalent to the coefficients $R_p^\R(\ca A,\ca B,\mu)$ for the scalar case
since 
\cite[Thm.~4.1]{BrBrJa87a} shows that for all $p\in [2,\infty]$ there exists a constant $c_p>0$ such that for all Hilbert spaces $H$ we have
\begin{equation}\label{grothendieck-cor}
R_p^\R(\ca A,\ca B,\mu) \Leq R_p^H(\ca A,\ca B,\mu) \Leq c_p \, R_p^\R(\ca A,\ca B,\mu)\, .
\end{equation}
Note that for $p=2$ we actually have $c_p=1$ and for $p=\infty$ we may choose the famous Grothendieck constant (see the proof of Lemma 2.2 in \cite{DePh82a}). Moreover, it is obvious from the definition that $R_p^H(\ca A,\ca B,\mu)$ is decreasing in $p$, i.e.
$$
R_p^H(\ca A,\ca B,\mu) \Leq R_q^H(\ca A,\ca B,\mu) \, , \qquad \qquad q\leq p.
$$
In particular this yields $R_\infty^H(\ca A,\ca B,\mu) \leq R_p^H(\ca A,\ca B,\mu) \leq R_2^H(\ca A,\ca B,\mu)$ for all $p\in [2,\infty]$.
Finally, Theorem 4.13 in \cite{Bradley05v1} gives the highly non-trivial relation
\begin{equation}\label{mix-rio}
R_p^\R(\ca A,\ca B,\mu) \Leq 2\pi \,  \a^{1-\frac 2 p} (\ca A,\ca B,\mu)\, \p_{\mathrm{sym}}^{\frac 2 p}(\ca A,\ca B,\mu)\, , \qquad \qquad p\in [2,\infty].
\end{equation}
In view of our consistency results we are mainly interested in the coefficients $R_p^H$. Note that 
with the help of the above inequalities these coefficients
can be
estimated by the typically more accessible coefficients $\a$ and $\p$. The coefficient $\b$, which can often 
(see \cite[Prop.~3.22]{Bradley05v1} for an exact statement)
be
computed by 
$$
\b(\ca A,\ca B,\mu) = \E_\mu \sup_{B\in \ca B}\bigl| \mu(B)-\E_\mu(B|\ca A)  \bigr|\, ,
$$ 
is mainly mentioned because it was used in earlier 
works (see e.g.~\cite{Vidyasagar02,LoKuSc06a}) on learning from dependent observations.

Let us now consider mixing coefficients and corresponding mixing notion for stochastic processes:

\begin{definition}
Let $(\Om,\ca A,\mu)$ be a probability space,
$Z$ be a measurable space, and $\ca Z:=(Z_i)_{i\geq 1}$ be a $Z$-valued stochastic process on $\Om$.
Furthermore, let $\xi$ be one of the above mixing coefficients.
For $i,j\geq 1$ we define the {\em $\xi$-bi-mixing coefficient of $\ca Z$} by
$$
\xi(\ca Z,\mu,i,j) :=  \xi\bigl(\s(Z_i),\s(Z_j),\mu\bigr)\, ,
$$
where $\s(Z_i)$ denotes the $\s$-algebra generated by $Z_i$.
Furthermore,  for $n\geq 1$ the $\xi$-mixing and $\bar\xi$-mixing coefficients of $\ca Z$ are  de\-fined by 
\begin{eqnarray*}
\xi(\ca Z,\mu,n) &  := & \sup_{i\geq 1}\xi( \ca Z,\mu,i,i+n)\\
\bar \xi(\ca Z,\mu,n) & := & \sup_{i\geq 1} \xi  \bigl(\s(Z_1,\dots,Z_i), \s(Z_{i+n},Z_{i+1+n},\dots),\mu  \bigr)\, ,
\end{eqnarray*}
respectively.
In addition, we say that the process $\ca Z$ is:
\begin{enumerate}
\item {\em $\xi$-mixing} with respect to $\mu$ if the $\xi$-mixing coefficients tend to 0, i.e.
	$$
	\lim_{n\to \infty}\xi(\ca Z,\mu,n)= 0\, .
	$$
\item {\em weakly $\xi$-mixing} with respect to $\mu$ if the $\xi$-mixing coefficients tend to 0 on average, i.e. 
	$$
	\lim_{n\to \infty}\frac  1{  n}  \sum_{k=1}^n  \xi(\ca Z,\mu,k)=0\, .
	$$
\item {\em weakly $\xi$-bi-mixing} with respect to $\mu$ if the $\xi$-bi-mixing coefficients tend to 0 on average, i.e. 
	\begin{equation}\label{amix-markov-h1}
	\lim_{n\to \infty}\frac  1{  n^{2}}  \sum_{i=1}^n\sum_{j=1}^{i-1} \xi(\ca Z,\mu,i,j)= 0\, .
	\end{equation}
\end{enumerate}
Finally, we define mixing notions analogous to {i)} and {ii)} for $\bar \xi$.
\end{definition}

It is immediately clear that $\xi(\ca Z,\mu,n) \leq \bar \xi(\ca Z,\mu,n)$, and consequently, every upper bound on $\bar \xi(\ca Z,\mu,n)$
translates into an upper bound on $\xi(\ca Z,\mu,n)$. This trivial observation is interesting since the literature typically deals with 
$\bar \xi(\ca Z,\mu,n)$, whereas the consistency results
which we will present in the following subsection 
 only require bounds on $\xi(\ca Z,\mu,n)$ or $\xi(\ca Z,\mu,i,j)$.
Finally, it is interesting to note that for stationary, homogeneous Markov chains $\ca Z$ we   actually have 
$\xi(\ca Z,\mu,n) = \bar \xi(\ca Z,\mu,n)$ for all $n\geq 1$ and $\xi\neq \p_{\mathrm{sym}}$.

Obviously, every  $\xi$-mixing process is weakly $\xi$-mixing, and since a simple induction over $n\in \N$ shows
$$
\sum_{i=1}^n\sum_{j=1}^{i-1} \xi(\ca Z,\mu,i,j) = \sum_{k=1}^{n-1}\sum_{m=1}^{n-k} \xi(\ca Z,\mu,m+k,m)\, , \qquad \qquad n\geq 1,
$$
we also see that every weakly $\xi$-mixing process is weakly $\xi$-bi-mixing. Moreover, 
if the process $\ca Z$ is  $\mu$-stationary in the wide sense then an elementary proof (see e.g.~\cite[Prop.~3.6]{Bradley05v1}) shows 
$\xi(\ca Z,\mu,i,j) = \xi(\ca Z,\mu,i+k,j+k)$ for all $i,j,k\geq 1$. Since  this implies
$\xi(\ca Z,\mu,i,j) = \xi(\ca Z,\mu,i-j+1)$ for $i\geq j\geq 1$ we then find
\begin{equation}\label{bi-mix-mix}
\sum_{i=1}^n\sum_{j=1}^{i-1} \xi(\ca Z,\mu,i,j) 
= 
\sum_{k=1}^{n-1}\sum_{m=1}^{n-k} \xi(\ca Z,\mu,m+k,m)
= 
\sum_{k=1}^{n-1} (n-k) \, \xi(\ca Z,\mu,k+1)
\end{equation}
for all $n\geq 1$. Consequently, every  stationary weakly $\xi$-bi-mixing process is actually weakly $\xi$-mixing.
%
%From the  relations presented  earlier between the different mixing coefficients we immediately obtain the following 
%implications
%\begin{displaymath}
%\begin{array}{rcccll}
%\p\mbox{-mixing } & \implies & \b \mbox{-mixing } & \implies & \a\mbox{-mixing }\\
%\p\mbox{-mixing } & \implies & R_p \mbox{-mixing } & \implies & \a\mbox{-mixing } \qquad \qquad& \mbox{for all } p\in[2,\infty] \\
%\a\mbox{-mixing } & \Longleftrightarrow & R_p \mbox{-mixing } &  & & \mbox{for all } p\in(2,\infty] ,
%\end{array}
%\end{displaymath}
%and the same implications obviously hold true for weak mixing, weak bi-mixing, and the $\bar\xi$-versions of all these mixing conditions.
Moreover, if the process $\ca Z$ is stationary and mixing in the sense of (\ref{mixing-general}), then \cite[Theorem 4.1]{Bradley05a} shows that
$\bar \b(\ca Z,\mu,n_0)<1$ or $\bar \p(\ca Z,\mu,n_0)<1$ for some $n_0\geq 1$ implies $\bar \b$-mixing or $\bar \p$-mixing, respectively.
Finally, it is discussed on \cite[p.~124]{Bradley05a} that  stationary processes $\ca Z$
with $\bar \p(\ca Z,\mu,n_0)<1/2$ for some $n_0\geq 1$
are $\bar \p$-mixing.

Examples of $\bar \xi$-mixing, and in particular $\bar \a$-mixing 
processes including certain Markov,  ARMA, MA($\infty$), and GARCH processes can be found in 
\cite[Sect.~2.6.1]{FaYa03} and \cite[p.~405ff]{Bradley05v1}.
Moreover, mixing properties of Gaussian processes are considered in \cite[Chapter 9]{Bradley05v1}. In particular, \cite[Theorem 9.5]{Bradley05v1}
shows $\bar \a(\ca Z,\mu,n) \leq \bar R_2^\R(\ca Z,\mu,n)\leq 2\pi \bar\a(\ca Z,\mu,n)$, $n\geq 1$, for stationary Gaussian processes.
Finally, \cite[Theorem 26.5]{Bradley05v3} together with \cite[Proposition 3.18]{Bradley05v1} 
shows that for all continuous, strictly decreasing functions $g:[0,\infty)\to (0,1/24)$
for which $x\mapsto \log g(x)$ is convex there exists a stationary process $\ca Z$
with $g(n)/4 \leq \bar \a(\ca Z,\mu,n) \leq \bar \p(\ca Z,\mu,n)\leq 4g(n)$ for all $n\geq 1$.
Note that this result in particular shows that in general the $\bar \xi$-mixing rates can be arbitrarily slow.
A brief survey of these and other results together with various references is given in \cite{Bradley05a}.

For Markov chains there are quite a few results on mixing coefficients 
(see e.g.~\cite{Bradley05a}, \cite[Chapter 7]{Bradley05v1}, and \cite[Chapter 21]{Bradley05v2}). Here we only 
recall the most important ones:
\cite[Theorem 7.5]{Bradley05v1} (see also \cite[Theorem 3.3]{Bradley05a}) shows that if a homogeneous Markov chain $\ca Z$ satisfies 
$R_2^\R(\ca Z,\mu,n_0)<1$ or $\p(\ca Z,\mu,n_0)<1/2$ for some $n_0\geq 1$ then 
$R_2^\R(\ca Z,\mu,n)$ or $\p(\ca Z,\mu,n)$ 
tend at least exponentially fast to 0, and by considering the proof it is also possible to derive explicit bounds for this convergence.
Moreover, if the Markov chain is also stationary, ergodic and aperiodic then $\p(\ca Z,\mu,n_0)<1$ suffices to obtain exponential $\p$-mixing rates.
In contrast,  for stationary Markov chains there are {\em no\/} similar results possible for $\b$-mixing coefficients (see e.g.~\cite[Theorem 21.3]{Bradley05v2}) or $\a$-mixing coefficients 
(see e.g.~\cite[Ex.~7.11]{Bradley05v1}).
Because of this lack previous learning results based on $\b$-mixing required 
rather strong additional assumptions on (stationary) Markov chains such as certain variants of geometric mixing conditions
(see e.g.~\cite[p.~100ff]{Vidyasagar02} and compare with \cite[Theorem 3.7]{Bradley05a} which shows that such geometric mixing conditions are
equivalent to exponentially fast $\b$-mixing).
Moreover, \cite[Theorem 3.4]{Bradley05a} shows that stationary, ergodic, and aperiodic Markov chains $\ca Z$ with 
$\a(\ca Z,\mu,n_0)<1/4$ for some $n_0\geq 1$
are automatically $\a$-mixing. Similarly, \cite[Corollary 3.6]{Bradley05a} shows that stationary, aperiodic Markov chains are
$\b$-mixing if and only if they are irreducible or Harris recurrent. 
Finally, stationary  Markov processes $\ca Z$
satisfying Doeblin's condition (\ref{doeblin}) satisfy $\p(\ca Z,\mu,n_0)<1$ for some $n_0\geq 1$
(see e.g.~\cite[p.~121]{Bradley05a}).
Further information on mixing properties of Markov chains can be found in \cite[Chapter 21]{Bradley05v2}.

Now let $(Z,\ca B,\mu)$ be a probability  space
and $\ca Z:= (T^{i-1})_{i\geq 1}$ be an invariant dynamical system on $Z$.  For $i\geq j\geq 1$ we then have $\s(Z_i)= \s(T^{i-1}) \subset \s(T^{j-1}) = \s(Z_j)$ and hence we obtain 
\begin{eqnarray*}
\a(\ca Z,\mu,i,j) 
%= 
%\sup_{\substack{A\in  \s(Z_i)\\ B\in \s(Z_j)}} \bigl|  \mu(A\cap B) - \mu(A)\mu(B)     \bigr|
\geq 
\sup_{A\in  \s(Z_i)} \bigl|  \mu(A\cap A) - \mu(A)\mu(A)     \bigr| 
%&=&
%\sup_{B\in \ca B} \bigl|  \mu(T^{i-1}B) - \mu^2(T^{i-1}B)     \bigr|\\
 = 
\sup_{B\in \ca B} \mu(B)\bigl(1-\mu(B)\bigr)\, .
\end{eqnarray*}
Consequently, $\ca Z$ is not weakly $\a$-bi-mixing if $\ca B$ is not $\mu$-trivial. However, note that {\em images} of dynamical systems
can even be strongly $\a$-mixing. Indeed, every i.i.d.~sequence is the image of an invariant dynamical system and the independence 
implies that all $\a$-coefficients are equal to 0.
For more information on ergodic mixing and its relation to $\xi$-mixing we refer to \cite[Chapter 22]{Bradley05v3} and \cite{Bradley05a}.

Let us finally discuss some laws of large numbers for mixing processes.
We begin with the 
following simple result which shows that asymptotically mean stationary, 
weakly bi-mixing processes satisfy the WLLNE:

\begin{proposition}\label{mix-implies-wlln}
Let $(\Om,\ca A,\mu)$ be a probability space,
$Z$ be a measurable space, and $\ca Z:=(Z_i)_{i\geq 1}$ be a $Z$-valued stochastic process on $\Om$
which is weakly $\a$-bi-mixing with respect to $\mu$. 
Then 
the following statements are equivalent:
\begin{enumerate}
\item $\ca Z$ is AMS. 
\item $\ca Z$ satisfies the WLLNE.
\end{enumerate}
\end{proposition}

For the quite simple proof of this proposition we refer to Section \ref{proofs}. Moreover, using \cite[Thm.~8.2.1]{Revesz68}
it is easy to see that for $\bar \a$-mixing processes AMS is actually equivalent to SLLNE.
Finally, \cite[Cor.~8.2.2]{ZhCh96} shows that identically distributed processes $\ca Z$ with 
\begin{equation}\label{phi-slln}
\sum_{n=1}^\infty \sqrt{\bar \p(\ca Z,\mu,2^n)}<\infty
\end{equation}
satisfy the SLLN. Note that in the above summability condition only a ``few'' $\bar \p$-coefficients are considered.
In particular, (\ref{phi-slln}) is satisfied whenever there are constants $c>0$ and $\a>2$ with   
$\bar \p(\ca Z,\mu,n)\leq c\, (\ln n )^{-\a}$ for all  $n\geq 2$.

\subsection{Consistency of SVMs for Mixing Processes}\label{mixing-consist}

In this subsection we establish consistency results for data-generating processes
with known upper bounds on the weakly $\a$-bi-mixing rate.
Unlike in the case of general processes satisfying a law of large numbers these new consistency results
give explicit conditions on the regularization sequences guaranteeing consistency.

In order to formulate these results we have to introduce a new quantity. To this end let $k$ be a bounded kernel
over some set $X$. Then the supremum norm of $k$ is defined by 
$$
\inorm k := \sup_{x\in X} \sqrt{k(x,x)}\, . 
$$
Note that the boundedness of $k$ implies $\inorm k<\infty$. Moreover, for the Gaussian kernels $k_\s$ we have 
$\inorm{k_\s}= 1$.

Now we can present our first consistency result which deals with locally Lipschitz-continuous loss functions:

\begin{theorem}\label{rate-learn}
Let $X$ be a separable metric space, $Y\subset \R$ be a closed subset
and  $L:X\times Y\times \R\to [0,\infty)$
be a convex, locally Lipschitz continuous loss function with $\inorm{L(.,.,0)}\leq c$.
Moreover, let 
$(\Om,\ca A,\mu)$ be a probability space, 
$\ca Z:=((X_i,Y_i))_{i\geq 1}$ be an $X\times Y$-valued, AMS stochastic process on $\Om$, and  $P$ be the asymptotic mean of $(\ca Z,\mu)$.
In addition,  let $H$ be an $(L,P)$-rich RKHS over $X$ with  bounded continuous kernel $k$.
We write 
$$
B_{\lb}:= \snorm k_\infty \Bigl({\frac{c}\lb}\Bigr)^{1/2}\, , \qquad \qquad \lb>0.
$$
Finally, assume that there are constants $C\in (0,\infty)$ and $\a\in (0,1]$ with
\begin{eqnarray}\label{rate-learn-h1}
\biggl| \frac  1 n \sum_{i=1}^n \E_\mu f\circ Z_i - \E_P f\biggr| & \leq & C \inorm fn^{-\a}\\ \label{rate-learn-h2}
\frac  1{  n^{2}}  \sum_{i=1}^n\sum_{j=1}^{i-1} \a(\ca Z,\mu,i,j) & \leq & C n^{-\a}
\end{eqnarray}
for all $f\in \sL \infty Z$ and all $n\geq 1$. Then for all 
null-sequence $(\lb_n)$ of strictly positive real numbers with
\begin{equation}\label{rate-learn-h3}
\frac {|L|_{B_{\lb_n},1}^4}{\lb_n^2n^\a}\to 0\
\end{equation}
the corresponding $(\lb_n)$-SVM based on $H$ and $L$ is  $L$-consistent for $\ca Z$.
\end{theorem}

The above result is of particular interest for binary classification problems. Indeed, recall that the standard SVM for classification
uses the hinge loss defined by 
$$
L(y,t):= \max\{0,1-yt\}\, , \qquad \qquad y\in Y:= \{-1,1\}, \, t\in \R.
$$
Obviously, this loss function is convex and 
Lipschitz continuous with $|L|_1 = 1$ and $L(y,0) = 1$ for $y\in Y$. 
For $X:= \Rn$ and $H_\s$ being the RKHS of a Gaussian RBF kernel with
fixed width $\s$ we consequently obtain $L$-consistency for the corresponding $(\lb_n)$-SVM whenever $\lb_n\to 0$ and $\lb_n^2 n^\a \to \infty$, 
where $\a$ is the exponent satisfying (\ref{rate-learn-h1}) and (\ref{rate-learn-h2}).
Since $L$-consistency implies binary classification consistency (see e.g.~\cite{Ste03d,BaJoMc03a}) we hence see that the above SVM 
is classification consistent. In particular, this consistency generalizes earlier  consistency results of \cite{Ste02a,Zhang04a,Ste03d}
with respect to both the compactness assumption on $X$ and the i.i.d.~assumption on the data-generating process.

In the case of $\a=1$ the SVMs using the hinge loss $L$ and 
an $(L,P)$-rich RKHS
is consistent if $\lb_n\to 0$ and $n\lb_n^2  \to \infty$. Since this is exactly the condition ensuring consistency in the i.i.d.~case 
we see that such an SVM is quite robust against violations of the i.i.d.~assumption.

If quantitative approximation properties of $H$ in terms of 
convergence rates for $\RP L \fP \to \RPB L$ are known, the proof Theorem \ref{rate-learn} also provides learning rates.
However, we conjecture that
these rates are usually overly conservative in terms of the confidence since 
we only employ 
Markov's inequality. Therefore  we do not discuss these convergence rates in further detail. Instead we would like to compare 
our consistency result with the consistency result for regularized boosting algorithms derived in \cite{LoKuSc06a}.
To this end we first observe that for
 (in the wide sense)
stationary processes (\ref{rate-learn-h1}) is automatically satisfied and 
(\ref{rate-learn-h2}) is equivalent to 
$$
\frac 1 n  \sum_{i=1}^n \a(\ca Z,\mu,i)  \Leq  C n^{-\a}\, , \qquad \qquad n\geq 1,
$$
by (\ref{bi-mix-mix}). Obviously, the latter is satisfied if $\ca Z$ 
is {\em algebraically $\bar \a$-mixing with exponent $\a$\/}, 
i.e.~if it satisfies $\bar \a(\ca Z,\mu,n)  \leq  C n^{-\a}$ for all $n\geq 1$. 
Consequently, Theorem \ref{rate-learn} implies consistency results for stationary, algebraically $\bar \a$-mixing
processes with known lower bound on the mixing rate. Compared to this  \cite{LoKuSc06a} only establishes a consistency result
for stationary, algebraically $\bar \b$-mixing
processes with known lower bound on the mixing rate. Since in general $\bar \a$-mixing is strictly weaker assumption 
than $\bar \b$-mixing we see that Theorem \ref{rate-learn} substantially weakens the assumptions of \cite{LoKuSc06a}.
Finally, note that our restriction to polynomial rates in (\ref{rate-learn-h1}) and (\ref{rate-learn-h2}) is by no means 
necessary. For example, if we replace $n^{-\a}$ by $(\log n)^{-\a}$ in (\ref{rate-learn-h1}) and (\ref{rate-learn-h2})
then the corresponding condition on $(\lb_n)$ for the SVM using the hinge loss becomes $\lb_n^2  (\log n)^{\a}\to \infty$.
In particular, note that such an SVM is consistent for {\em all\/} stationary, algebraically $\a$-mixing processes!\footnote{
However, for such
$(\lb_n)$ the SVM  typically deals too conservatively  with the stochastic part of the learning process,
so that the approximation behaviour is poor. As a consequence this result does not seem to have any practical relevance.}
In this direction it is interesting to recall that in \cite{Irle97a} consistency was established for kernel estimators
and algebraically $\a$-mixing, not necessarily stationary processes. To our best knowledge this is the 
consistency result that is closest in its assumptions on $\ca Z$ to Theorem \ref{rate-learn}.

The proof of Theorem \ref{rate-learn} is based on a stability argument together with a simple Markov-type 
concentration inequality for Hilbert space valued random variables.
In principle, one could also employ exponential type inequalities for sums of $\R$-valued random variables
in the sense of e.g.~\cite[Chapter 1.4]{Bosq98} together with a skeleton argument based on e.g.~covering numbers.
However, our preliminary considerations showed that the resulting conditions on $(\lb_n)$ were 
substantially stronger, and hence we do not discuss this approach in further detail.

%\fix{Hier noch uniform Bound \& Bosq Hoeffding diskutieren: polynomial and Gauss mit Zhou}
%\fix{Markov chains Bsp: $R_2(n_0)<1$ implies exponential convergence. However, $\a$-mixing is not sufficient, \cite[Ex 7.11]{Bradley05v1}
%for the latter. Frage: Ist $\a$-mixing sufficient for exponential rate of $R_p(n)$ for $p>2$? This obviously reduces to 
%question whether \cite[Thm 7.5]{Bradley05v1} holds for $R_p(n)$ for $p>2$. }

The next theorem establishes a result similar to Theorem \ref{rate-learn} for distance-based loss functions of some growth type $p$:

\begin{theorem}\label{rate-learn2}
Let $L:\R\times \R\to [0,\infty)$ be a convex distance-based loss function of upper growth type $p\in [1,2]$.
Furthermore, let $X$ be a separable metric space and
$H$ be an $(L,P)$-rich RKHS over $X$ with  bounded continuous kernel $k$.  
Moreover, let 
$(\Om,\ca A,\mu)$ be a probability space, 
$\ca Z:=((X_i,Y_i))_{i\geq 1}$ be an $X\times \R$-valued, AMS stochastic process on $\Om$, and  $P$ be the asymptotic mean of $(\ca Z,\mu)$.
Assume that we have 
\begin{equation}\label{rate-learn2-h10}
\sup_{i\geq 1} |\mu_{(X_i,Y_i)}|_q \, < \, \infty
\end{equation}
for some $q\in [p,\infty]$, where $|.|_q$ is defined by (\ref{moment-def}).
Furthermore assume that there are constants $C>0$ and $\a,\b\in (0,1]$ with
\begin{eqnarray}\label{rate-learn2-h1}
\biggl| \frac  1 n \sum_{i=1}^n \E_\mu f\circ Z_i - \E_P f\biggr| & \leq & C \snorm f_{\Lx 1 P}\, n^{-\a}\\ \label{rate-learn2-h2}
\frac  1{  n^{2}}  \sum_{i=1}^n\sum_{j=1}^{i-1} \a^{1-\frac{2p-2}q}(\ca Z,\mu,i,j) \p^{\frac{2p-2}q}_{\mathrm{sym}}(\ca Z,\mu,i,j) & \leq & C n^{-\b}
\end{eqnarray}
for all $f\in \Lx 1P \cap \bigcap_{i=1}^\infty \Lx 1 {\mu_{(X_i,Y_i)}}$. Then for all 
null-sequences $(\lb_n)$ of strictly positive real numbers with
\begin{eqnarray}\label{rate-learn2-h3}
\lb_n^p n^{2\a} & \to &\infty \\ \label{rate-learn2-h3xx}
\lb_n^{2p} n^{\b} & \to &\infty
\end{eqnarray}
the corresponding $(\lb_n)$-SVM based on $H$ and $L$ is  $L$-consistent for $\ca Z$.
\end{theorem}

Since distance based loss functions are typically used for regression problems we see that the above theorem 
is mainly interesting for these learning scenarios. For Lipschitz continuous losses such as 
the absolute distance loss $L(y,t):= |y-t|$, the $\epsilon$-insensitive loss $L(y,t):= \max\{0,|y-t|-\epsilon\}$, the 
logistic loss or Huber's robust loss we obviously have $p=1$ and hence (\ref{rate-learn2-h2}) reduces to (\ref{rate-learn-h2}).
Moreover, for Lipschitz continuous losses we can choose $q=1$ in (\ref{rate-learn2-h10}). Consequently, it is easy to see that all remarks 
made for the classification SVM using the hinge loss, remain true for regression SVMs using one of the above losses.

In contrast to this the least squares SVM which uses the standard least squares loss $L(y,t):= (y-t)$
requires $p=2$ in the above theorem. For processes with uniformly bounded noise, i.e.~$q=\infty$,
we  again see that (\ref{rate-learn2-h2}) reduces to (\ref{rate-learn-h2}). Moreover, for $q\in (2,\infty)$ we have 
$$
\frac  1{  n^{2}}  \sum_{i=1}^n\sum_{j=1}^{i-1} \a^{1-\frac{2}q}(\ca Z,\mu,i,j) \p_{\mathrm{sym}}^{\frac{2}q}(\ca Z,\mu,i,j)  
\Leq 
\biggl(\frac  1{  n^{2}}  \sum_{i=1}^n\sum_{j=1}^{i-1} \a(\ca Z,\mu,i,j)\biggr)^{1-\frac{2}q}
$$
so that (\ref{rate-learn-h2}) implies (\ref{rate-learn2-h2}) for $\b:= \a(1-2/q)$. However, for $q=2$ we have ${1-\frac{2p-2}q}= 0$, and consequently
we only obtain consistency results for weakly $\p_{\mathrm{sym}}$-bi-mixing processes.

Theorem \ref{rate-learn2} generalizes the only known consistency result (see \cite{ChSt04b}) for regression SVMs dealing with unbounded noise 
with respect to both the compactness assumption on $X$ and the i.i.d.~assumption on the data-generating process.
In particular,  Theorem \ref{rate-learn2}  shows that such SVMs are rather robust against violations of these assumptions, and consequently
it gives a strong justification of using such SVMs in rather general situations.

Finally, we like to mention that condition (\ref{rate-learn2-h10}) can be replaced by a weaker assumption 
describing the average behaviour of the sequence $(|\mu_{(X_i,Y_i)}|_q)_{i\geq 1}$. However, the resulting conditions 
on $(\lb_n)$ are more complicated and hence we omit the details.

\section{Proofs}\label{proofs}

\subsection{Proofs from Subsection \ref{s-process}}

\begin{proofof}{Lemma \ref{limit-distrib}}
Let $\ca B$ be the $\s$-algebra of $Z$. 
We write $P_n (B):= \frac 1 n \sum_{i=1}^n \mu(Z_i\in B)$ for  $B\in \B$ and $n\geq 1$.
Then $P_n$ is obviously a probability measure on $\ca B$ for all $n\geq 1$. Now 
the theorem of Vitali-Hahn-Saks (see e.g.~\cite[p.~158-160]{DuSc88}) ensures that $P(B):= \lim_{n\to \infty} P_n(B)$, $B\in \ca B$, defines a probability measure on $\ca B$.
\end{proofof}

\begin{proofof}{Theorem \ref{exists-limit-distrib-th}}
Recall that the convergence in probability $\mu$ can  be described by the metric
$$
d(f,g) := \int_\Om \min\bigl\{ 1, |f-g|\bigr\}\, d\mu\, , \qquad \qquad f,g\in \sL 0\Om.
$$
Moreover, for measurable $B\subset Z$ let $c_B$ be the constant satisfying (\ref{wlln-events}). 
The  WLLNE and the above metric then shows 
$$
\lim_{n\to \infty}\int_\Om \,\,\,\Bigl|\frac 1 n  \sum_{i=1}^n \eins_B\circ Z_i - c_B  \Bigr|\, d\mu = 0\, .
$$
Since $\snorm ._{\Lxx 1 \mu}$ is continuous on $\Lxx 1 \mu$ we hence find 
%By \cite[Satz 21.7]{Bauer90} we then find 
$$
\lim_{n\to \infty}\frac 1 n \sum_{i=1}^n \E_\mu\eins_B\circ Z_i 
=
\lim_{n\to \infty} \int_\Om \frac 1 n \sum_{i=1}^n \eins_B\circ Z_i d\mu
= 
\lim_{n\to \infty} \int_\Om \,\,\,  \Bigl| \frac 1 n \sum_{i=1}^n \eins_B\circ Z_i \Bigr| \, d\mu
= 
\E_\mu |c_B| 
= 
c_B\, ,
$$
where the existence of the right limit implies the existence of the left limit.
Consequently,  $\ca Z$ is AMS and we have  $P(B) = c_B$. Obviously, the latter together with (\ref{wlln-events})
immediately gives (\ref{slln-events-ams}). Finally, if $\ca Z$ satisfies the SLLNE then we obtain
the almost sure convergence in (\ref{slln-events-ams}) from  (\ref{slln-events}).
\end{proofof}

\begin{proofof}{Lemma \ref{lln-bounded}}
Let us begin by showing the assertion for the strong law.
To this end 
we fix an $\e>0$.
By the approximation lemma for bounded measurable functions %see e.g. \cite[p.126, Floret81}
there  exists a step function $g:X\to \R$ with $\inorm{f-g}\leq \e$.
Now, the linearity of the limit shows 
$$
\E_P g= \lim_{n\to \infty}\frac 1 n \sum_{i=1}^n g \circ Z_i(\om)
$$
for $\mu$-almost all $\om \in \Om$, and consequently, \cite[Lemma 20.6]{Bauer01} gives an $n_0\geq 1$ with 
\begin{equation}\label{slln-bounded-h1}
\mu\biggl(\sup_{n\geq n_0}\Bigl| \frac 1 n \sum_{i=1}^n g\circ Z_i - \E_P g  \Bigr|\leq \e   \biggr)\geq 1-\e\, .
\end{equation}
Moreover, for $\om\in \Om$ we have
\begin{eqnarray*}
&&
\sup_{n\geq n_0}\Bigl| \frac 1 n \sum_{i=1}^n f\circ Z_i(\om) - \E_P f  \Bigr|\\
& \leq &
\sup_{n\geq n_0}\Bigl| \frac 1 n \sum_{i=1}^n f\circ Z_i(\om) - \frac 1 n \sum_{i=1}^n g\circ Z_i (\om)  \Bigr| 
+ \Bigl| \frac 1 n \sum_{i=1}^n g\circ Z_i(\om) - \E_P g  \Bigr| 
+ \bigl| \E_P g-\E_P f \bigr|\\
& \leq & 2\e + \sup_{n\geq n_0} \Bigl| \frac 1 n \sum_{i=1}^n g\circ Z_i(\om) - \E_P g  \Bigr| \, ,
\end{eqnarray*}
and hence we obtain
$$
\mu\biggl(\sup_{n\geq n_0}\Bigl| \frac 1 n \sum_{i=1}^n f\circ Z_i - \E_P f  \Bigr|\leq 3 \e   \biggr)\geq 1-\e\, .
$$
This shows the $\mu$-almost sure convergence in (\ref{lln-bounded-hx}).
Using that the functions $\frac 1n \sum_{i=1}^n f\circ Z_i$, $n\geq 1$, are uniformly bounded Lebesgue's theorem then yields
$$
\E_P f = \int_\Om \E_P f\, d\mu 
= 
\int_\Om \lim_{n\to \infty}\frac 1 n \sum_{i=1}^n  f\circ Z_i d\mu
= 
 \lim_{n\to \infty}\int_\Om\frac 1 n \sum_{i=1}^n f \circ Z_i d\mu
=
 \lim_{n\to \infty}\frac 1 n \sum_{i=1}^n \E_\mu f\circ Z_i\, ,
$$
and hence we have found (\ref{lln-bounded-h0}). 
Finally, if $\ca Z$ satisfies the WLLNE then pulling the supremum out of $\mu$ in (\ref{slln-bounded-h1}) and adjusting the rest of the 
proof accordingly shows (\ref{lln-bounded-hx}) with convergence in probability $\mu$. Moreover, in this case (\ref{lln-bounded-h0})
can be shown analogously to the argument used in the proof Theorem \ref{exists-limit-distrib-th}.
\end{proofof}

%%%%%%%%%%%%%%%%%%%%%%%%%%%%%%%%%%%%%%%%%%%%%%%%%%%%%%%%%%%%%%%%5555

% example sec

%%%%%%%%%%%%%%%%%%%%%%%%%%%%%%%%%%%%%%%%%%%%%%%%%%%%%%%%%%%%%%%%%%%%

\subsection{Proofs from Subsection \ref{s-process-ex}}

\begin{proofof}{Proposition \ref{uncor-wllne}}
$ii) \Rightarrow i)$. Follows from Theorem \ref{exists-limit-distrib-th}.\\
$i) \Rightarrow ii)$. 
Let $P$ be the stationary mean of $(\ca Z,\mu)$. Then there exists an $n_0\geq 1$ such that 
$$
\Bigl| \frac 1 n\sum_{i=1}^n \E_\mu \eins_B\circ Z_i - P(B)\Bigr| < \frac \e 2\, , \qquad \qquad n\geq n_0.
$$
For $n\geq n_0$ Markov's inequality then yields
\begin{eqnarray*}
&&
\mu\biggl( \Bigl\{   \om \in \Om: \Bigl|   \frac 1 n \sum_{i=1}^n \eins_B\circ Z_i(\om) - P(B)\Bigr|\geq \e\Bigr\}\biggr) \\
& \leq &
\mu\biggl( \Bigl\{   \om \in \Om: \Bigl|   \frac 1 n \sum_{i=1}^n \eins_B\circ Z_i(\om) - \frac 1 n\sum_{i=1}^n \E_\mu \eins_B\circ Z_i\Bigr|\geq \frac\e 2\Bigr\}\biggr)\\
%+
%\mu\biggl( \Bigl\{   \om \in \Om: \Bigl|   \sum_{i=1}^n \E_\mu \eins_B\circ Z_i - P(B)\Bigr|\geq \frac\e 2\Bigr\}\biggr)
&\leq &
4 \e^{-2} n^{-2} \E_\mu \Bigl(  \sum_{i=1}^n \bigl(\eins_B\circ Z_i - \E_\mu \eins_B\circ Z_i  \bigr)  \Bigr)^2\, .
\end{eqnarray*}
Let us write $h_i :=  \eins_B\circ Z_i - \E_\mu \eins_B\circ Z_i$, $i\geq 1$. 
Then we have $\E_\mu h_i = 0$ and $h_i(\om)\in [-1,1]$ for all $i\geq 1$ and all $\om \in \Om$. Moreover, for $i\neq j$ we have $\E h_i h_j = 0$ since we assume that 
$\eins_B \circ Z_i$ and $\eins_B\circ Z_j$ are uncorrelated. Consequently, we obtain
$$
\E_\mu \Bigl(  \sum_{i=1}^n \bigl(\eins_B\circ Z_i - \E_\mu \eins_B\circ Z_i  \bigr)  \Bigr)^2 \leq n\, ,
$$
from which we easily obtain the assertion.
\end{proofof}

\begin{proofof}{Proposition \ref{marting-slln}}
Let us define $Y:= f\circ Z_1$ and $X_n:= \frac 1 n \sum_{i=1}^n f\circ Z_i$, $n\geq 1$.
Then (\ref{marting-slln-h1}) states $\E(X_{n-1}\, |\, \ca F_n) = X_n$ for all $n\geq 2$, and hence we obtain
$$
X_n = \E(X_{n-1}\, | \, \ca F_n) = \E\bigl(\E(X_{n-2}\, |\, \ca F_{n-1}) \, |\, \ca F_n\bigr) = \E(X_{n-2}\, |\, \ca F_n) = \ldots = \E(X_1\, | \ca F_n) 
= \E(Y|\ca F_n)
$$
for all $n\geq 2$. Moreover, $X_1$ is $\ca F_1$-measurable and hence we also have $X_1 = \E(X_1\, | \ca F_1) 
= \E(Y|\ca F_1)$. Now, \cite[Theorem 6.6.3]{AsDD00} shows that $\lim_{n\to \infty}X_n = E Y$ almost surely.
Furthermore, from $X_n =  \E(Y|\ca F_n)$, $n\geq 1$, we also conclude $\E_\mu X_n = \E_\mu Y= \E_\mu f\circ Z_1$, and hence 
$\mu_{Z_1}$ is the asymptotic mean of $(\ca Z,\mu)$. Combining these results then gives the assertion.
\end{proofof}

\begin{proofof}{Proposition \ref{noisy-DS}}
Without loss of generality we may assume that $\ca E$ is of canonical form, i.e.~$\e_i = \pi_1 \circ S^{i-1}$, $i\geq 1$,
where $\pi_1: (\Rn)^\N\to \Rn$ is the first coordinate projection, $S$ is the shift operator on $(\Rn)^\N$, and $\nu$ is a product measure,
i.e.~$\nu= (\mu')^\N$ for a suitable measure $\mu'$ on $\Rn$.
Then 
$\ca S:= (S^{i-1})_{i\geq 1}$ is weakly mixing, and consequently \cite[p.~65]{Petersen89}
shows that $\ca Z\times \ca S$ is $\mu\otimes \nu$-ergodic.
By Theorem \ref{ergodic-char} we can then conclude that $\ca Z\times \ca S$ satisfies the  $\mu\otimes \nu$-SLLN.
Moreover, we have $T^{n-1}+\e_n = T^{n-1}+\pi_1 \circ S^{n-1}$ and hence  $\ca Z+\ca E$ 
is an image of the process $\ca Z\times \ca S$. From this we easily conclude that $\ca Z+ \ca E$ satisfies the  $\mu\otimes \nu$-SLLN.
\end{proofof}

\subsection{Proofs from Subsection \ref{loss}}

Before we prove Lemma \ref{risk-confirm} we first have to recall the following  elementary lemma whose proof is omitted:

\begin{lemma}\label{average-trick}
Let $(a_i)$ be a sequence of real numbers and $a\in \R$ such that 
$$
\lim_{n\to \infty} \frac 1 n\sum_{i=1}^n a_i \Eq a\, .
$$
Then for all $n_0\geq 0$ we have 
$$
\lim_{n\to \infty}\frac 1 {n-n_0} \sum_{i=n_0+1}^n a_i \Eq a\, .
$$
\end{lemma}

\begin{proofof}{Lemma \ref{risk-confirm}}
Let us first assume that $L$ is locally bounded. 
By Lemma \ref{average-trick} it then suffices to consider the case $n_0=0$. Now observe that the function $g(x,y):= L(x,y,f(x))$, $(x,y)\in X\times Y$,
is a bounded, measurable function since $f$ is assumed to be bounded, and $L$ is locally bounded. 
Applying  Lemma \ref{lln-bounded} to the function $g$ then gives the assertion.\\
Let us now assume that $L$ is a $P$-integrable Nemitski loss.
Then there exists an $b\in \Lxx 1 P$ and an increasing function $h:[0,\infty)\to [0,\infty)$
with 
$$
g(x,y) \Leq b(x,y) + h\bigl( \inorm f\bigr)\, , \qquad \qquad (x,y)\in X\times Y.
$$
This shows $g\in \Lxx 1 P$, and hence the assertion follows from Definition \ref{lln-l1-def}.
\end{proofof}

%%%%%%%%%%%%%%%%%%%%%%%%%%%%%%%%%%%%%%%%%%%%%%%%%%%%%%%%%%%%%%%%5555

% Loss Sec

 %%%%%%%%%%%%%%%%%%%%%%%%%%%%%%%%%%%%%%%%%%%%%%%%%%%%%%%%%%%%%%%%%%%%

%%%%%%%%%%%%%%%%%%%%%%%%%%%%%%%%%%%%%%%%%%%%%%%%%%%%%%%%%%%%%%%%5555

% svm sec

%%%%%%%%%%%%%%%%%%%%%%%%%%%%%%%%%%%%%%%%%%%%%%%%%%%%%%%%%%%%%%%%%%%%

\subsection{Proofs from Section \ref{consist-svm}}

For the proof of Theorem \ref{asymp-learn} we need some preparations. Let us begin with the following 
result  on the existence and uniqueness of infinite sample SVMs which is a slight extension of similar results 
established in  \cite{DeRoCaPiVe04a,ChSt04b}:

\begin{theorem}\label{infinite:exist-and-unique}
Let $L:X\times Y\times \R\to [0,\infty)$ be a convex  loss function
and $P$ be a distribution on $X\times Y$ such that $L$ is a $P$-integrable Nemitski loss.
Furthermore, let $H$ be a RKHS of a bounded measurable kernel over $X$. Then 
for all $\lb>0$ there exists
exactly one element $\fP\in H$
such that 
\begin{equation}\label{infinite:exist-and-unique-h1}
\lb \snorm\fP_{H}^{2} + \RP L \fP = \inf_{f\in H} \lb \snorm f_{H}^{2} + \RP L f\, .
\end{equation}
Furthermore, we have $\snorm \fP_{H}\leq \sqrt{\frac{\RP L 0}\lb}$.
\end{theorem}

The following two results describe the stability of the empirical SVM solutions. 
The first result was (essentially) shown in \cite{DeRoCaPiVe04a,ChSt04b}:

\begin{theorem}\label{infinite:general-repres-co1}
Let $X$ be a separable metric space,
$L:X\times Y\times \R\to [0,\infty)$ be a convex, locally Lipschitz continuous loss function,
and $P$ be a distribution on $X\times Y$ with $\RP L 0<\infty$.
Furthermore,
let $H$ be the RKHS of a bounded, continuous kernel $k$ over $X$ with canonical feature map $\P:X\to H$.
We define 
$$
B_{\lb}:= \snorm k_\infty \biggl({\frac{\RP L 0}\lb}\biggr)^{1/2}\, , \qquad \qquad \lb>0.
$$
Then for all $\lb>0$
there exists a bounded, measurable function $h_\lb:X\times Y\to \R$ with 
\begin{equation}\label{infinite:lipschitz-h1}
\inorm{h_\lb} \leq |L|_{B_{\lb},1}
\end{equation}
and 
\begin{equation}\label{infinite:lipschitz-h0}
\mnorm{\fP-\fT}_{H} \Leq \frac 1 \lb \mnorm{\E_P h_\lb\Phi- \E_{T}h_\lb\Phi}_{H}
\end{equation}
for all training sets $T=((x_1,y_1),\dots,(x_n,y_n))\in (X\times Y)^n$, where $\E_T$ denotes the expectation operator with respect to the 
empirical measure associated to $T$, i.e.~$\E_T g := \frac 1 n \sum_{i=1}^n g(x_i,y_i)$.
\end{theorem}

Recall that convex distance-based loss functions are in general {\em not\/} locally Lipschitz continuous.
Nevertheless SVM using these losses still enjoy stability as the following result shows:

\begin{theorem}\label{infinite:general-repres-co2}
Let $X$ be a separable metric space,
$L:\R\times \R\to [0,\infty)$ be a convex, distance-based loss function of upper growth type $p\geq  1$ and 
$P$ a distribution on $X\times \R$ with $|P|_{q}<\infty$
for some $q\in[p,\infty]$. Furthermore, 
let
$H$ be a RKHS of a bounded, continuous kernel over $X$ with canonical feature map $\P:X\to H$.
Then there exists a constant $c_{L}>0$  depending only on $L$ such that 
for all $\lb>0$
there exists a measurable function $h_\lb:X\times Y\to \R$ with 
\begin{eqnarray}\label{infinite:general-repres-co2-h2xx}
\snorm {h_\lb}_{\Lxx s{\bar P}} & \leq  & 8^{p}c_{L} \,\bigl( 1+ |\bar P|_{q}^{p-1} + \snorm\fP_\infty^{p-1} \bigr)\\
\label{infinite:lipschitz-h3}
\mnorm{\fP-\fT}_{H} & \leq & \frac 1 \lb \mnorm{\E_P h_\lb\Phi- \E_{T}h_\lb\Phi}_{H}
\end{eqnarray}
for $s:= \frac q{p-1}$, all distributions $\bar P$ on $X\times \R$ with $|\bar P|_q<\infty$ and all training sets $T\in (X\times Y)^n$.
Finally, if $L$ is also of lower growth type $p$ then we additionally have 
\begin{equation}
\label{infinite:general-repres-co2-h2}
\snorm {h_\lb}_{\Lxx s{P}}  \Leq   16^{p}c_{L} \,\bigl( 1+ |P|_{q}^{p-1}  \bigr)\Bigl(1 + \snorm\fP_\infty^{\frac{q-p}s}  \Bigr)\, .
\end{equation}
\end{theorem}

\begin{proof}
By taking care in the constants in the proof of \cite[Theorem 13]{ChSt04b} we obtain a measurable function $h_\lb:X\times Y\to \R$
satisfying (\ref{infinite:lipschitz-h3}) and $$
|h_\lb(x,y)|
\leq
4^{p }\,c_{L} \max\bigl\{ 1, |y-f_{P,\lambda}(x)|^{p-1}\bigr\}\, , \qquad \qquad (x,y)\in X\times Y,
$$
where $c_{L}$ is a suitable constant depending  only on the loss function $L$.
For $q=\infty$ we then easily find the assertion, and hence let us assume $q\in[p,\infty)$.
In this case, the above inequality yields
\begin{equation}\label{infinite:general-repres-co2-h2xxx}
\bigl|  h_\lb(x,y) \bigr|^{s} 
\leq 
4^{p s}c_{L}^{s} \max\bigl\{ 1, |y-f_{P,\lambda}(x)|^{q}\bigr\}
\leq  
4^{p s} 2^{q-1} c_{L}^{s} \Bigl(  1 + |y|^{q} + |\fP(x)|^{q}  \Bigr)\, .
%& \leq & 
%2^{q-1}\bigl(c_{L}')^{s} \Bigl(  1 + |y|^{q} + \inorm\fP^{q-p}  |\fP(x)|^{p}  \Bigr)\, ,
\end{equation}
Since $\frac{q-1}s\leq p$ and $s\geq 1$ we then obtain (\ref{infinite:general-repres-co2-h2xx}).
Moreover, if $\psi$ is the function satisfying $L(y,t) = \psi(y-t)$, $y,t\in \R$, we have 
\begin{eqnarray*}
\E_{P} |\fP|^{p}
& \leq &
2^{p-1}\int_{X\times Y}  \bigl|y-\fP(x)\bigr|^{p} + |y|^{p} \, dP(x,y)\\
& \leq &
2^{p-1} \int_{X\times Y}  c_{L}^{(1)}\,\psi\bigl(y-\fP(x)\bigr)+1 + |y|^{p} \, dP(x,y)\\
& = &
2^{p-1}  \Bigl( c_{L}^{(1)} \Rx L {P}\fP + 1 + | P|_{p}^{p}  \Bigr)\\
& \leq &
2^{p-1}  \Bigl( c_{L}^{(1)} \Rx L {P} 0 + 1 + | P|_{p}^{p}  \Bigr)\\
& \leq &
2^{p-1}  \Bigl( c_{L}^{(2)} \bigl( 1 + | P|_{p}^{p} \bigr) + 1 + | P|_{p}^{p}  \Bigr)\\
&\leq&
2^{p}c_{L}^{(3)} \bigl( 1 + | P|_{p}^{p} \bigr)\, ,
\end{eqnarray*}
where $c_{L}^{(1)}$, $c_{L}^{(2)}\geq 1$, and $c_{L}^{(3)}\geq 1$ are
suitable constants depending  only on the loss function $L$.
Combining the estimate on $\E_{P} |\fP|^{p}$ with (\ref{infinite:general-repres-co2-h2xxx}) then gives
\begin{eqnarray*}
\snorm {h_\lb}_{\Lx s{P}} 
& \leq &
4^p \,2^{\frac{q-1}s} c_{L} \Bigl(  1 +  |P|_{q}^{p-1} + \inorm\fP^{\frac{q-p}s} \bigl(\E_{P} |\fP|^{p} \bigr)^{\frac  1s} \Bigr)\\
& \leq &
4^p\, 2^{\frac{q-1}s} c_{L} \Bigl(  1 +  |P|_{q}^{p-1} + \inorm\fP^{\frac{q-p}s} \bigl(2^{p}c_{L}^{(3)} ( 1 + | P|_{p}^{p} ) \bigr)^{\frac  1s} \Bigr)\\
& \leq &
4^p\, 2^{\frac{p+q}s} \bigl(c_{L}^{(4)}\bigr)^{1+\frac 1 s}
\,\bigl( 1+|P|_{p}^{\frac{p}s}+ |P|_{q}^{p-1}  \bigr)\bigl(1 + \inorm\fP^{\frac{q-p}s}  \bigr)\, ,
\end{eqnarray*}
where $c_{L}^{(4)}\geq 1$ is another 
suitable constant depending  only on the loss function $L$.
Now note that we have $\frac{p+q}s = (\frac  pq +1)(p-1)\leq 2(p-1)$ and  $1+\frac 1s\leq 2 $. These estimates together with 
$$
|P|_{p}^{\frac p s} \Leq |P|_{q}^{\frac p s} \Eq |P|_{q}^{\frac {p(p-1)}q} \Leq 1 + |P|_{q}^{p-1}
$$
then yield (\ref{infinite:general-repres-co2-h2}). %The proof of (\ref{infinite:general-repres-co2-h2xx}) is similiar though easier.
\end{proof}

The next lemma establishes  Hilbert space valued laws of large numbers which are later
used to bound the term $\mnorm{\E_P h_\lb\Phi- \E_{T}h_\lb\Phi}_{H}$.

\begin{lemma}\label{lln-Hilbert}
Let $(\Om,\ca A,\mu)$ be a probability space,
$Z$ be a Polish space, and $\ca Z:=(Z_i)_{i\geq 1}$ be a   $Z$-valued stochastic process on $\Om$.
Assume that $\ca Z$ satisfies the WLLNE and  let $P$ be the asymptotic mean of $(\ca Z,\mu)$.
Furthermore, let  $H$ be a Hilbert space, and $\P:Z\to H$ be a continuous and bounded map. 
Then for all $h\in \Lxx \infty P$ we have 
\begin{equation}\label{slln-hilbert}
	\lim_{n\to \infty}\frac 1 n \sum_{i=1}^n (h\P) \circ Z_i = \E_P h\P\, ,
\end{equation}
where the convergence is in probability $\mu$. Moreover, if $\ca Z$ actually satisfies the WLLN then 
(\ref{slln-hilbert}) holds for all $f\in \Lx 1 P$. Finally, the convergence holds $\mu$-almost surely 
for all $f\in \Lx \infty P$ or $f\in \Lx 1 P$ 
if $\ca Z$
satisfies the SLLNE or SLLN, respectively.
\end{lemma}

\begin{proof}
Let us first show (\ref{slln-hilbert}) for $f\in \Lx 1 P$ when $\ca Z$ satisfies the SLLN.
To this end we first make the additional assumption that 
there exists a compact subset $K\subset Z$ with $h(z)=0$ for all $z\not\in K$. 
Now recall that $\P$ is continuous and hence $\P(K)\subset H$ is compact. 
Moreover, recall that $H$ as a Hilbert space has the approximation property 
(see e.g.~\cite[p.~30ff]{LiTz77} for details
on this concept). 
For a fixed $\e>0$ there consequently 
exists a bounded linear operator $S:H\to H$ with $m:=\rank S<\infty$ and 
$$
\snorm{S\P(z)-\P(z)}_H\leq \e\, , \qquad \qquad z\in K.
$$
Let $e_1,\dots,e_m$ be an ONB of the image  $SH$ of $H$ under $S$. Since $\langle e_j, S\P\rangle:Z\to \R$, $j=1,\dots,m$, are  
bounded measurable functions we then find that 
$$
\langle e_j,hS\P\rangle = h\langle e_j, S\P\rangle\, , \qquad \qquad j=1,\dots,m,
$$
are $P$-integrable.
Consequently, they satisfy the limit relation (\ref{slln-l1}),
and by a well-known reformulation of almost sure convergence (see e.g.~\cite[Lem.~20.6]{Bauer01})
there hence exists an $n_\e$ such that with probability not less than $1-\e$ we have both
$$
\sup_{n\geq n_\e} \sup_{j=1,\dots,m} 
\biggl| \frac 1 n \sum_{i=1}^n \bigl\langle e_j,hS\P\bigr\rangle \circ Z_i (\om) - \E_P \langle e_j,hS\P\rangle \biggr| \leq \e m^{-1/2}
$$
and 
$$
\sup_{n\geq n_\e} \biggl| \frac 1 n \sum_{i=1}^n |h|\circ Z_i(\om) - \E_P |h|\biggr| \leq \e\, .
$$
Let us fix an  $n\geq n_\e$ and an $\om\in \Om$ which satisfies these two inequalities. 
Using $h(z)=0$ for all $z\in Z\mysetminus K$ we then have 
\begin{eqnarray*}
\norm{\frac 1 n \sum_{i=1}^n (h\P) \circ Z_i(\om) - \frac 1 n \sum_{i=1}^n (hS\P) \circ Z_i(\om)}_H
& \leq &
\frac 1 n\sum_{i=1}^n |h| \circ Z_i(\om) \cdot \snorm{\P\circ Z_i(\om) - S\P\circ Z_i(\om)}_H \\
& \leq &
\frac \e n \sum_{i=1}^n |h|\circ Z_i(\om)\\
& \leq &
\e \bigl( \e + \E_P |h|\bigr) \\
& \leq &
\e + \e\, \E_P|h|\, .
\end{eqnarray*}
Moreover, $n$ and $\om$ also satisfy
\begin{eqnarray*}
\norm{\frac 1 n \sum_{i=1}^n (hS\P) \circ Z_i(\om) - \E_P hS\P}_H
& = &
\biggl( \sum_{j=1}^m \Bigl|  \Bigl\langle e_j,\frac 1 n \sum_{i=1}^n (hS\P) \circ Z_i(\om) - \E_P hS\P\Bigr\rangle            \Bigr|^2          \biggr)^{1/2}\\
& \leq &
\sqrt m \sup_{j=1,\dots,m}   \biggl| \frac 1 n \sum_{i=1}^n \bigl\langle e_j,hS\P\bigr\rangle \circ Z_i (\om) - \E_P \langle e_j,hS\P\rangle \biggr|\\
& \leq & 
\e\, .
\end{eqnarray*}
In addition, $h(z)=0$ for all $z\in Z\mysetminus K$ implies
$$
\mnorm{ \E_P hS\P - \E_P h\P }_H \leq \int_K |h(z)| \cdot\snorm{S\P(z)-\P(z)}_H\, dP(z) \leq \e\, \E_P|h|\, ,
$$ 
and consequently we can conclude
\begin{eqnarray*}
\norm{\frac 1 n \sum_{i=1}^n (h\P) \circ Z_i(\om) - \E_P h\P}_H
& \leq &
\norm{\frac 1 n \sum_{i=1}^n (h\P) \circ Z_i(\om) - \frac 1 n \sum_{i=1}^n (hS\P) \circ Z_i(\om)}_H\\
&&\quad
+ 
\norm{\frac 1 n \sum_{i=1}^n (hS\P) \circ Z_i(\om) - \E_P hS\P}_H
+ 
\mnorm{ \E_P hS\P \!-\! \E_P h\P }_H\\
& \leq &
2\e\bigl(1+\E_P|h|\bigr)\, .
\end{eqnarray*}
This shows 
$$
\mu\biggl(\biggl\{\om\in \Om:  \sup_{n\geq n_\e} \norm{\frac 1 n \sum_{i=1}^n (h\P) \circ Z_i(\om) - \E_P h\P}_H \leq 2\e\bigl(1+\E_P|h|\bigr)\biggr\}\biggr) \geq 1-\e\, ,
$$
and hence \cite[Lemma 20.6]{Bauer01} yields the  assertion for our special case.\\
Let us now prove the assertion for general $h\in \Lx 1 P$. To this end
we may assume without loss of generality that $\snorm{\P(z)}\leq 1$ for all $z\in Z$.
Let us fix an $\e>0$. Since $Z$ is Polish the measures 
$P$ and $|h|P$ are regular and hence there then exists a compact subset $K\subset Z$ with 
$$
P(Z\mysetminus K)\leq \e
\qquad \qquad \mbox{ and } \qquad \qquad 
\int_{Z\setminus K} |h| \, dP \leq \e\, . 
$$
Now $g:= \eins_K h$ is a $P$-integrable function that vanishes outside the compact set $K$.
Our preliminary considerations and the SLLN consequently show that 
there exists an $n_\e\geq 1$ such that with probability not less than $1-\e$ we have both 
$$
\sup_{n\geq n_\e} \norm{\frac 1 n \sum_{i=1}^n (g\P) \circ Z_i(\om) - \E_P g\P}_H \leq \e
$$
and
$$
\sup_{n\geq n_\e} \biggl|\frac 1 n \sum_{i=1}^n \bigl(\eins_{Z\setminus K}|h|\bigr) \circ Z_i(\om) - \E_P \eins_{Z\setminus K}|h| \biggr| \leq \e
$$
Let us fix an  $n\geq n_\e$ and an $\om\in \Om$ which satisfies these two inequalities. 
Using $h-g= \eins_{Z\setminus K} h$ and  $\snorm{\P(z)}\leq 1$ for all $z\in Z$
we then obtain
\begin{eqnarray*}
\norm{\frac 1 n \sum_{i=1}^n (h\P) \circ Z_i(\om) - \E_P h\P}_H
&\leq &
\norm{\frac 1 n \sum_{i=1}^n (h\P) \circ Z_i(\om) - \frac 1 n \sum_{i=1}^n (g\P) \circ Z_i(\om)}_H\\
&&
\quad +\,
\norm{\frac 1 n \sum_{i=1}^n (g\P) \circ Z_i(\om) - \E_P g\P}_H
+
\mnorm{\E_P g\P - \E_P h\P}_H\\
& \leq &
\frac 1 n \sum_{i=1}^n \bigl(\eins_{Z\setminus K}|h|\bigr) \circ Z_i(\om) + \e + \E_P \eins_{Z\setminus K}|h|\\
& \leq &
\e + \E_P \eins_{Z\setminus K}|h| + \e + \E_P \eins_{Z\setminus K}|h|\\
& \leq &
4\e\, .
\end{eqnarray*}
Therefore we obtain 
$$
\mu\biggl( \biggl\{\om\in \Om: \sup_{n\geq n_\e} \norm{\frac 1 n \sum_{i=1}^n (h\P) \circ Z_i(\om) - \E_P h\P}_H \leq 4\e\biggr\}\biggr) \geq 1-\e\, ,
$$
and hence we obtain the assertion by another application of  \cite[Lemma 20.6]{Bauer01}.\\
Finally, if $\ca Z$  only satisfies the WLLN then we obtain the assertion by omitting the terms $\sup_{n\geq n_\e}$
in the above proof.
Moreover, for processes satisfying only a law of large numbers for events we have to use Lemma \ref{lln-bounded}
instead of Definition \ref{lln-l1-def}. 
\end{proof}

In order to prove Theorem \ref{asymp-learn} we finally  need the following technical lemma:

\begin{lemma}\label{auswahl}
Let $F:(0,\infty)\times \N\to [0,\infty)$ be a function with $\lim_{n\to \infty}F(\lb,n)=0$ for all $\lb>0$.
Then there exists a sequence $(\lb_n)\subset (0,1]$ with 
$$
\lim_{n\to \infty} \lb_n = 0
$$
and 
$$
\lim_{n\to \infty} F(\lb_n,n) = 0\, .
$$
\end{lemma}

\begin{proof}
For $k\geq 1$ there exists an $n_k\geq 1$ such that for all $n\geq n_k$ we have 
\begin{equation}\label{lemma-h1}
F(k^{-1},n) < k^{-1}\, .
\end{equation}
Obviously, we may assume without loss of generality that $n_k< n_{k+1}$ for all $k\geq 1$. For $n\geq 1$ we write
$$
\lb_n
:=
\begin{cases}
1 & \mbox{ if } 1 \leq n < n_1\\
k^{-1} & \mbox{ if } n_k\leq n < n_{k+1}\, .
\end{cases}
$$
Now let $\e>0$. Then there exists an integer $k\geq 1$ with $k^{-1}\leq \e$. Let us fix an $n\geq n_k$. Then there exists an $i\geq k$ 
with $n_i\leq n < n_{i+1}$, and consequently we have $\lb_n = i^{-1}$.
This gives 
$$
\lb_n = i^{-1} \leq k^{-1} \leq \e\, ,
$$
and since (\ref{lemma-h1}) together with $n_i\leq n$ yields $F(i^{-1},n) \leq i^{-1}$ we also find 
$$
F(\lb_n,n) = F(i^{-1},n) \leq i^{-1} \leq \e\, .
$$
These estimates show the assertion.
\end{proof}

\begin{proofof}{Theorem \ref{asymp-learn}}
We only show the assertion in the case of $\ca Z$ satisfying the SLLNE. 
Since $L$ is locally bounded, the function $L(.,.,0)$ is bounded and hence we may assume without loss of generality that $\Rx L Q 0\leq 1$
for all distributions $Q$ on $X\times Y$. By a standard argument this assumption leads to
$$
\snorm \fQ_H \Leq    \lb^{-1/2}
$$
for {\em all\/} distributions $Q$ on $X\times Y$ and all $\lb>0$.
Moreover, we may assume without loss of generality that $\inorm k\leq 1$, so that we have $\inorm f\leq \snorm f_H$ for all $f\in H$.
Now, let us fix an $\e>0$. 
Since a  simple argument shows 
that $\lim_{\lb\to 0}\RP L \fP = \RPxB LH = \RPB L$ we then find 
\begin{eqnarray*}
\Bigl| \RP L \fTnon  - \RPB L \Bigr| 
&\leq  &
\Bigl| \RP L \fTnon  - \RP L \fP \Bigr| + \Bigl| \RP L \fP - \RPB L\Bigr|\\
& \leq &
|L|_ {\lb^{-1/2},1}\,\, \, \inorm {\fTnon - \fP} + \e \\
& \leq &
\frac {|L|_ {\lb^{-1/2},1}} {\lb} \,\,\,\mnorm{ \E_{T_n(\om)}h_\lb\Phi-\E_P h_\lb\Phi}_{H} + \e 
\end{eqnarray*}
for all $n\geq 1$, $\om \in \Om$, and
all sufficiently small $\lb>0$, where $h_\lb:X\times Y\to \R$ is the function according to Theorem \ref{infinite:general-repres-co1}, and 
$\E_{T_n(\om)}$ denotes the expectation operator with respect to the empirical distribution associated to the training set
$T_n(\om)= ((X_1(\om),Y_1(\om)),\dots,(X_n(\om),Y_n(\om)))$, 
i.e.~$\E_{T_n(\om)} g = \frac 1 n \sum_{i=1}^n g(X_i(\om),Y_i(\om))$.
Furthermore, for all $\lb\in(0,\e]$ and $n\geq 1$ we have
\begin{eqnarray*}
&&
\mu \biggl(\Bigl\{  \om\in \Om:  \sup_{m\geq n} \frac {|L|_ {\lb^{-1/2},1}} {\lb} \mnorm{ \E_{T_m(\om)}h_\lb\Phi-\E_P h_\lb\Phi}_{H} \geq  \e \Bigr\}   \biggr)\\
& \leq &
\mu \biggl(\Bigl\{  \om\in \Om:  \sup_{m\geq n} \mnorm{ \E_{T_m(\om)}h_\lb\Phi-\E_P h_\lb\Phi}_{H} \geq  \frac{\lb^2}   {|L|_ {\lb^{-1/2},1}}  \Bigr\} \biggr)\\
& =: & F(\lb,n)\, .
\end{eqnarray*}
Moreover, by Theorem \ref{infinite:general-repres-co1} we know that $h_\lb$ is a bounded function for  all $\lb>0$ and consequently, 
Lemma \ref{lln-Hilbert} yields $\lim_{n\to \infty}F(\lb,n)= 0$ for all $\lb\in (0,\e]$. Now Lemma \ref{auswahl} shows that there exists a sequence $(\lb_n)$
with $\lb_n\to 0$ and $F(\lb_n,n)\to 0$. For fixed $\d>0$ there consequently exists an $n_0\geq 1$ such that for all $n\geq n_0$ we have
$|\RP L \fPn - \RPB L|\leq \e$, $\lb_n\leq \e$, and $F(\lb_n,n)\leq \d$. For such $n$ our previous considerations then show
\begin{eqnarray*}
&&
\mu\biggl( \Bigl\{ \om\in \Om:  \sup_{m\geq n}\Bigl| \RP L \fTnso  - \RPB L \Bigr| \geq 2\e  \Bigr\}\biggr) \\
& \leq &
\mu\biggl(\Bigl\{ \om\in \Om: \sup_{m\geq n} \frac {|L|_ {\lb_{m}^{-1/2},1}} {\lb_m}\,\,\, \mnorm{ \E_{T_m(\om)}h_{\lb_{m}}\Phi-\E_P h_{\lb_{m}}\Phi}_{H} \geq \e  \Bigr\}\biggr) \\
& \leq &
F(\lb_n,n)\\
& \leq &
\d\, .
\end{eqnarray*}
This shows the assertion.
\end{proofof}

\begin{proofof}{Theorem \ref{asymp-learn2}}
Again, we only show the assertion in the case of $\ca Z$ satisfying the SLLN. 
Obviously, we may assume without loss of generality that $\inorm k\leq 1$, so that we have $\inorm f\leq \snorm f_H$ for all $f\in H$.
Moreover, since $|P|_p<\infty$ we may additionally assume without loss of generality that  both  $|P|_p\leq 1$ and $\RP L0\leq 1$.
Note that the latter assumption  immediately yields
$$
\snorm\fP_H \leq \lb^{-1/2}
$$
for all $\lb>0$. Let $\psi:\R\to [0,\infty)$ be the function satisfying $L(y,t)=\psi(y-t)$, $y,t\in \R$.
The assumption $|P|_p<\infty$ then guarantees $\psi \in \Lx 1 P$ and hence 
the SLLN shows 
\begin{equation}\label{asym-h1}
\lim_{n\to \infty} \Rx L {T_n(\om)}0 = \lim_{n\to \infty}\E_{T_n(\om)}\psi = \E_P \psi = \RP L 0
\end{equation}
for $\mu$-almost all $\om\in \Om$. Moreover, we have $\lb \snorm \fTnon_H^2 \leq \Rx L {T_n(\om)}0$ for all $n\geq 1$, $\lb>0$, 
and $\om \in \Om$, and 
consequently the ``local Lipschitz continuity'' of the $L$-risk established in 
\cite[Lemma 25]{ChSt04b} together with Theorem \ref{infinite:general-repres-co2} yields
\begin{eqnarray*}
&&
\bigl| \RP L \fTnon - \RP L \fP\bigr| \\
&\leq &
c_p \Bigl( |P|_{p-1} + \snorm\fTnon_\infty^{p-1} + \snorm\fP_\infty^{p-1}+1     \Bigr) \inorm{\fTnon-\fP}\\
& \leq &
\frac{c_p}\lb \biggl( 2 + \Bigl( \frac{\Rx L {T_n(\om)}0}{\lb}    \Bigr)^{\frac{p-1}2} + \lb^{-\frac{p-1}2}  \biggr) \snorm{\E_{T_n(\om)}h_\lb\Phi-\E_P h_\lb\Phi}_H
\end{eqnarray*}
for all $n\geq 1$, $\lb>0$, 
and $\om \in \Om$. Let us fix an $\e>0$. For $\lb\in(0,\e]$ and $n\geq 1$ we then obtain
\begin{eqnarray*}
&&
\mu \biggl(\Bigl\{  \om\in \Om:  \sup_{m\geq n} \bigl| \RP L \fTnom - \RP L \fP\bigr|  \geq  \e \Bigr\}   \biggr)\\
& \leq &
\mu \biggl(\!\biggl\{  \om\in \Om:  \sup_{m\geq n} \biggl( 2 + \Bigl( \frac{\Rx L {T_m(\om)}0}{\lb}    \Bigr)^{\frac{p-1}2} + \lb^{-\frac{p-1}2}  \biggr)\snorm{\E_{T_m(\om)}h_\lb\Phi-\E_P h_\lb\Phi}_H  \geq  \frac{\lb^2}{c_p} \biggr\} \!  \biggr)\\
& =: \!\!& F(\lb,n)\, .
\end{eqnarray*}
Moreover, Theorem \ref{infinite:general-repres-co2} ensures $h_\lb \in \Lx 1 P$ for all $\lb>0$ and hence Lemma \ref{lln-Hilbert} together with 
(\ref{asym-h1}) shows $\lim_{n\to \infty}F(\lb,n)= 0$ for all $\lb\in (0,\e]$.
Now the rest of the proof is analogous to the proof of Theorem \ref{asymp-learn}.
\end{proofof}

%%%%%%%%%%%%%%%%%%%%%%%%%%%%%%%%%%%%%%%%%%%%%%%%%%%%%%%%%%%%%%%%5555

% svm sec

%%%%%%%%%%%%%%%%%%%%%%%%%%%%%%%%%%%%%%%%%%%%%%%%%%%%%%%%%%%%%%%%%%%%

\subsection{Proofs from Subsection \ref{mixing-coeff}}

\begin{proofof}{Proposition \ref{mix-implies-wlln}}
$ii) \Rightarrow i)$. Follows from Theorem \ref{exists-limit-distrib-th}.\\
$i) \Rightarrow ii)$. Let $P$ be the stationary mean of $(\ca Z,\mu)$. 
As in the proof of Proposition \ref{uncor-wllne} we then find an 
$n_0\geq 1$ such that 
for all $n\geq n_0$ we have 
$$
\mu\biggl( \Bigl\{   \om \in \Om: \Bigl|   \frac 1 n \sum_{i=1}^n \eins_B\circ Z_i(\om) - P(B)\Bigr|\geq \e\Bigr\}\biggr) 
\leq 
4 \e^{-2} n^{-2} \E_\mu \Bigl(  \sum_{i=1}^n \bigl(\eins_B\circ Z_i - \E_\mu \eins_B\circ Z_i  \bigr)  \Bigr)^2\, .
$$
Let us write $h_i :=  \eins_B\circ Z_i - \E_\mu \eins_B\circ Z_i$, $i\geq 1$. 
Then we have $\E_\mu h_i = 0$ and $h_i(\om)\in [-1,1]$ for all $i\geq 1$ and all $\om \in \Om$. 
Consequently, (\ref{mix-rio}) gives
$R_\infty^\R(\ca Z,\mu,i,j) \leq 2\pi \a(\ca Z,\mu,i,j)$, $i,j\geq 1$, and hence we obtain
$$
\E_\mu \Bigl(  \sum_{i=1}^n \bigl(\eins_B\circ Z_i- \E_\mu \eins_B\circ Z_i  \bigr)  \Bigr)^2
 = 
\E_\mu \sum_{i=1}^n h_i^2 
+ 
2 \E_\mu \sum_{i=1}^n\sum_{j=1}^{i-1} h_ih_j
 \leq 
n + 4\pi  \sum_{i=1}^n\sum_{j=1}^{i-1} \a(\ca Z,\mu,i,j)\, .
$$
Combining the estimates then yields the assertion.
\end{proofof}

\subsection{Proofs from Subsection \ref{mixing-consist}}

\begin{proofof}{of Theorem \ref{rate-learn}}
Let $\ca B$ be the $\s$-algebra of $Z$. 
We write $P_n (B):= \frac 1 n \sum_{i=1}^n \mu(Z_i\in B)$ for  $B\in \B$ and $n\geq 1$.
Then $P_n$ is obviously a probability measure on $\ca B$ for all $n\geq 1$.
Let us first show that 
\begin{equation}\label{rate-learn-h4}
\lim_{n\to \infty} \RP L \fPnn = \RPB L\, .
\end{equation}
To this end we first observe that the assumption (\ref{rate-learn-h1}) yields
\begin{eqnarray} \nonumber
\RP L \fPnn 
& \leq & 
\lb_n \snorm \fPnn_H^2+\Rx L {P_n}\fPnn + C \inorm{L\circ \fPnn} n^{-\a}\\ \nonumber
& \leq &
\lb_n \snorm \fPn_H^2+\Rx L {P_n}\fPn + C \inorm{L\circ \fPnn} n^{-\a}\\ \label{rate-learn-h5}
& \leq &
\lb_n \snorm \fPn_H^2+\RP L \fPn + C n^{-\a} \bigl( \inorm{L\circ \fPn} + \inorm{L\circ \fPnn}  \bigr)
\end{eqnarray}
for all $n\geq 1$. Now  $\RPxB LH = \RPB L$ together with $\lb_n\to 0$ yields $\lb_n \snorm \fPn_H^2+\RP L \fPn\to \RPB L$.
Moreover, for every distribution $Q$ on $Z$ we have
$$
\inorm{L\circ \fQ} \Leq c + |L|_{\inorm\fQ,1}\inorm\fQ \Leq c + |L|_{B_{\lb},1} B_{\lb}
$$
by (\ref{loc-lip-imply-nemits}) and Theorem \ref{infinite:exist-and-unique}. In addition, 
$(|L|_{B_{\lb_n},1})$ is a non-decreasing sequence and the sequence $(B_{\lb_n})$ is 
dominated by the sequence $(\lb_n^{-1/2})$. Consequently, 
(\ref{rate-learn-h3}) implies 
$n^{-\a}|L|_{B_{\lb_n},1} B_{\lb_n}\to 0$ and hence we find (\ref{rate-learn-h4}).
Let us now fix an $\e>0$. Then Theorem \ref{infinite:general-repres-co1} and Markov's inequality yield
\begin{eqnarray*}
&&
\mu \biggl(\Bigl\{  \om\in \Om:  \bigl| \RP L \fTno - \RP L \fPnn\bigr|  \geq  \e \Bigr\}   \biggr)\\
& \leq &
\mu \biggl(\Bigl\{  \om\in \Om:  |L|_{B_{\lb_n},1}\,\, \inorm{\fTno-\fPnn}   \geq  \e \Bigr\}   \biggr)\\
& \leq &
\mu \biggl(\Bigl\{  \om\in \Om:  \inorm k |L|_{B_{\lb_n},1}\,\, \mnorm{ \E_{T_n(\om)}h_n\Phi-\E_{P_n} h_n\Phi}_{H}  \geq  \e \lb_n \Bigr\}   \biggr)\\
& \leq &
\frac{\snorm k_\infty^2  |L|_{B_{\lb_n},1}^2}{\e^2 \lb_n^2} \E_{\om\sim \mu}\mnorm{ \E_{T_n(\om)}h_n\Phi-\E_{P_n} h_n\Phi}_{H}^2
\end{eqnarray*}
where $h_n$ is the function according to Theorem \ref{infinite:general-repres-co1} for the distribution $P_n$ and the regularization parameter $\lb_n$.
Let us define 
$$
g_{n,i} := (h_n\P) \circ (X_i,Y_i) - \E_\mu  (h_n\P) \circ (X_i,Y_i)
$$  
for $n\geq 1$ and $i=1,\dots,n$.
Then we have $\E_\mu g_{n,i} = 0$ and Theorem \ref{infinite:general-repres-co1} yields
$$
\inorm{g_{n,i}} 
\Leq 
2 \sup_{\om \in \Om} \snorm{(h_n \P)\circ (X_i,Y_i)(\om)}_H
\Leq 2\, \inorm{h_n}\, \inorm k
\Leq 2\, \inorm k |L|_{B_{\lb_n},1}\, .
$$
Consequently, (\ref{mix-rio}) and (\ref{grothendieck-cor}) show that there exists a universal constant $c\geq 1$ such that
\begin{eqnarray*}
&&
\E_{\om\sim \mu}\mnorm{ \E_{T_n(\om)}h_n\Phi-\E_{P_n} h_n\Phi}_{H}^2\\
& = & 
{n^{-2}} \,\E_{\om\sim \mu} \norm{\sum_{i=1}^n (h_n \P)\circ (X_i,Y_i)(\om) - \E_\mu (h_n \P)\circ (X_i,Y_i)}_H^2\\
& = &
n^{-2} \sum_{i=1}^n \E_\mu \langle g_{n,i},g_{n,i}\rangle + 2 n^{-2} \sum_{i=1}^n \sum_{j=1}^{i-1}\E_\mu \langle g_{n,i},g_{n,j}\rangle\\
& \leq &
n^{-2} \sum_{i=1}^n \snorm{g_{n,i}}_\infty^2  
+ 2 n^{-2} \sum_{i=1}^n \sum_{j=1}^{i-1}R_\infty^H(\ca Z,\mu,i,j)\inorm{g_{n,i}}\inorm{g_{n,j}} \\
& \leq & 
 4 n^{-1} \snorm k_\infty^2 |L|_{B_{\lb_n},1}^2 + c\, \snorm k_\infty^2 |L|_{B_{\lb_n},1}^2 n^{-2} \sum_{i=1}^n \sum_{j=1}^{i-1}\a(\ca Z,\mu,i,j)
\end{eqnarray*}
for all $n\geq 1$. By combining all estimates and using (\ref{rate-learn-h3}) we then obtain the assertion.
\end{proofof}

\begin{proofof}{Theorem \ref{rate-learn2}}
Without loss of generality we assume $\inorm k\leq 1$ and $ |\mu_{(X_i,Y_i)}|_q \leq 1$ for all $i\geq 1$.
In addition, we can obviously, also assume $\lb_n\in (0,1]$ for all $n\geq 1$.
Now, we define $P_n (B):= \frac 1 n \sum_{i=1}^n \mu(Z_i\in B)$ for  measurable $B\subset  X\times \R$ and $n\geq 1$.
For $r\in [1,q]$ a simple calculation then shows 
\begin{equation}\label{rate-lern-hxxx}
|P_n|_r^r 
= 
\int_{X\times \R} |y|^r dP_n(x,y) 
= 
\frac 1 n \sum_{i=1}^n \int_{X\times \R} |y|^r d\mu_{(X_i,Y_i)}(x,y) 
= 
\frac 1 n \sum_{i=1}^n |\mu_{(X_i,Y_i)}|_r^r \leq 1\, .
\end{equation}
Moreover, \cite[Thm.~23.8]{Bauer01} together with Fatou's lemma yields
\begin{eqnarray*}
|P|_r^r 
\Eq
\int_0^\infty P\bigl(\{ (x,y)\in X\times \R: |y|^r\geq t        \}\bigr)dt
&=&
\int_0^\infty \lim_{n\to \infty} \frac 1 n\sum_{i=1}^n \mu \bigl(\{ \om\in\Om: |Y_i(\om)|^r\geq t        \}\bigr)dt\\
&\leq &
\liminf_{n\to \infty} \int_0^\infty  \frac 1 n\sum_{i=1}^n \mu \bigl(\{ \om\in\Om: |Y_i(\om)|^r\geq t        \}\bigr)dt\\
&\leq&
\liminf_{n\to \infty} \frac 1 n \sum_{i=1}^n |\mu_{(X_i,Y_i)}|_r^r \\
&\leq&  1\, .
\end{eqnarray*}
Having finished these preparations we can now begin with the actual proof. To this end  first observe that we obtain
$$
\RP L \fPnn 
\Leq
\lb_n \snorm \fPn_H^2+\RP L \fPn + C n^{-\a} \bigl( \snorm{L\circ \fPn}_{\Lx 1 P} + \snorm{L\circ \fPnn}_{\Lx 1 P}  \bigr)
$$
as in (\ref{rate-learn-h5}). Moreover, we obviously have $\snorm{L\circ \fPn}_{\Lx 1 P} = \RP L \fPn \leq \RP L 0 \leq c$ for some constant $c$ independent
of $n$. In addition, (\ref{rate-lern-hxxx}) yields
\begin{eqnarray*}
\snorm{L\circ \fPnn}_{\Lx 1 P}
& = &
 \int_{X\times Y} \psi\bigl( y-\fPnn(x)\bigr) dP(x,y)\\
& \leq & 
\tilde c_p  \int_{X\times Y} 1 +|y|^p + |\fPnn(x)|^p  dP(x,y)\\
& \leq & 
2\tilde c_p + \tilde c_p \snorm{\fPnn}_\infty^p\\
& \leq &
2\tilde c_p + \tilde c_p \snorm k_\infty^p \biggl(\frac {\Rx L{P_n} 0}{\lb_n}  \biggr)^{\frac p 2}\\
& \leq & 
2 c_p + c_p   \lb_n^{-\frac p2}\, ,
\end{eqnarray*}
where $\tilde c_p$ and $c_p$ are  constants only depending on $L$ and $p$. Combining these estimates with
 $\lim_{\lb\to 0}\RP L \fP = \RPxB LH = \RPB L$ and
 (\ref{rate-learn2-h3}) we then obtain 
$\lim_{n\to \infty} \RP L \fPnn = \RPB L$.\\
Now let us assume that we have an $\om \in \Om$ and an $n\geq 1$ with $\snorm{\fTno-\fPnn}_H\leq 1$.
For $p>1$
a simple calculation using \cite[Lemma 25]{ChSt04b} and $\lb_n\leq 1$ then shows 
\begin{eqnarray*}
&&
\bigl| \RP L \fPnn - \RP L \fTno  \bigr|\\
& \leq &
C_p \,\Bigl( |P|_{p-1}^{p-1} + \snorm\fPnn_\infty^{p-1} + \snorm\fTno_\infty^{p-1}+1 \Bigr)\, \inorm{\fPnn-\fTno}\\
&\leq &
C_p \,\Bigl( 2 + 2\snorm\fPnn_\infty^{p-1} + \snorm{\fTno-\fPnn}_\infty^{p-1} \Bigr)\, \snorm{\fPnn-\fTno}_H \\
& \leq &
C_p \,\Biggl( 3 + 2\biggl( \frac{\Rx L{P_n} 0}{\lb_n} \biggr)^{\frac {p-1}2}   \Biggr) \,\snorm{\fPnn-\fTno}_H \\
& \leq &
\bar C_p \, \lb_n^{-\frac{p-1}2}  \,\snorm{\fPnn-\fTno}_H \\
& \leq &
\bar C_p \, \lb_n^{-\frac{p+1}2} \mnorm{ \E_{T_n(\om)}h_n\Phi-\E_{P_n} h_n\Phi}_{H}   \, ,
\end{eqnarray*}
where $C_p\geq 1$ and $\bar C_p\geq 1$ are constants only depending on $p$ and $L$, and 
$h_n$ is the function according to Theorem \ref{infinite:general-repres-co1} for the distribution $P_n$ and the regularization parameter $\lb_n$.
Moreover, for $p=1$ we see that $L$ is Lipschitz continuous by \cite[Lemma 4]{ChSt04b} and hence 
the above estimate is also true in this case.
Let us now define 
$$
g_{n,i} := (h_n\P) \circ (X_i,Y_i) - \E_\mu  (h_n\P) \circ (X_i,Y_i)
$$ 
for $n\geq 1$ and $i=1,\dots,n$.
Then we have $\E_\mu g_{n,i} = 0$ and for $s:= \frac q{p-1}$ we find
\begin{eqnarray*}
\snorm{g_{n,i}}_{\Lx s \mu} 
\Leq
2 \snorm{h_n}_{\Lx s {\mu_{(X_i,Y_i)}}} 
& \leq &
128c_{L} \,\bigl( 1+ |\mu_{(X_i,Y_i)}|_{q}^{p-1} + \snorm\fPnn_\infty^{p-1} \bigr) \\
& \leq & 
128c_{L} \,\Biggl( 2 + \biggl( \frac{\Rx L{P_n} 0}{\lb_n} \biggr)^{\frac {p-1}2}   \Biggr)\\
& \leq &
C_{L,p} \, \lb_n^{-\frac{p-1}2}\, ,
\end{eqnarray*}
where $C_{L,p}>0$ is  a constant only depending on $L$ and $p$. For $\d>0$  
Markov's inequality together with $s\geq 2$, (\ref{mix-rio}) and (\ref{grothendieck-cor})
thus yields
\begin{eqnarray*}
&&
\mu \biggl(\Bigl\{  \om\in \Om:   \mnorm{ \E_{T_n(\om)}h_n\Phi-\E_{P_n} h_n\Phi}_{H}  \geq  \d  \Bigr\}   \biggr)\\
& \leq & 
\frac 1{\d^2n^2} \biggl(\sum_{i=1}^n \E_\mu \langle g_{n,i},g_{n,i}\rangle + 2  \sum_{i=1}^n \sum_{j=1}^{i-1}\E_\mu \langle g_{n,i},g_{n,j}\rangle\biggr)\\
& \leq &
\frac 1{\d^2n^2} \biggl(\sum_{i=1}^n  \snorm{g_{n,i}}_{\Lx s \mu}^2+ 2  \sum_{i=1}^n \sum_{j=1}^{i-1}R_s^H(\ca Z,\mu,i,j) \snorm{g_{n,i}}_{\Lx s \mu}\snorm{g_{n,j}}_{\Lx s \mu}\biggr)\\
& \leq &
\frac {\bar C_{L,p} }{\d^2 \lb_n^{p-1}n} + \frac {\bar C_{L,p} }{\d^2 \lb_n^{p-1}n^2} \sum_{i=1}^n\sum_{j=1}^{i-1} \a^{1-\frac{2p-2}q}(\ca Z,\mu,i,j) \p_{\mathrm{sym}}^{\frac{2p-2}q}(\ca Z,\mu,i,j)  \\
& \leq &
\frac {(1+C)\bar C_{L,p} }{\d^2 \lb_n^{p-1}n^\b}\, ,
\end{eqnarray*}
where $\bar C_{L,p}>0$ is  another  constant only depending on $L$ and $p$. 
Let us now fix an $\e\in (0,1]$.
For 
$\om \in \Om$ and  $n\geq 1$  with 
$$
\mnorm{ \E_{T_n(\om)}h_n\Phi-\E_{P_n} h_n\Phi}_{H} < \frac {\e\lb_n^{(p+1)/2}}{\bar C_p}
$$
we then have $\snorm{\fTno-\fPnn}_H< \frac {\e\lb_n^{(p-1)/2}}{\bar C_p}\leq 1$, and consequently we can conclude
\begin{eqnarray*}
&&
\mu \biggl(\Bigl\{  \om\in \Om:   \bigl| \RP L \fPnn - \RP L \fTno  \bigr| <  \e  \Bigr\}   \biggr)\\
& \geq &
\mu \biggl(\Bigl\{  \om\in \Om:   \mnorm{ \E_{T_n(\om)}h_n\Phi-\E_{P_n} h_n\Phi}_{H} < \frac {\e\lb_n^{(p+1)/2}}{\bar C_p}  \Bigr\}   \biggr)\\
& \geq &
1 - \frac {(1+C)\bar C_{L,p} \bar C_p^2 }{\e^2 \lb_n^{2p}n^\b}\, .
\end{eqnarray*}
Using (\ref{rate-learn2-h3xx}) then yields the assertion.
\end{proofof}

\bibliographystyle{unsrt}

\small{
%\bibliography{svm}
\bibliography{../../../literatur-DB/svm}
}

\end{document}